\documentclass[letterpaper, 10 pt, conference]{ieeeconf} 
\IEEEoverridecommandlockouts

\usepackage{enumitem}
\setlist{topsep=0pt, leftmargin=*}
\usepackage{algorithm}
\usepackage{algorithmicx}

\usepackage[noend]{algpseudocode}
\usepackage{amsmath, amsfonts, bm, amssymb, tabularx}
\usepackage{latexsym, mathtools}
\usepackage{amsthm}
\usepackage{ifthen}
\usepackage{hyperref}\hypersetup{colorlinks=true, unicode=true, linkcolor=[rgb]{0.10,0.05,0.67}, citecolor=[rgb]{0.10,0.05,0.67}, filecolor=[rgb]{0.10,0.05,0.67}, urlcolor=[rgb]{0.10,0.05,0.67}}
\usepackage{Shorthands}
\usepackage{cleveref}
\newcommand{\citep}[1]{\cite{#1}}

\usepackage[numbers, sort&compress]{natbib}
\allowdisplaybreaks

\title{\LARGE \bf
Regret Analysis of Multi-task Representation Learning for Linear-Quadratic Adaptive Control
}

\author{Bruce D.\ Lee$^{*,1}$, Leonardo F.\ Toso$^{*,2}$,
Thomas T.C.K.\ Zhang$^{*,1}$, 
James Anderson$^{2}$, and
Nikolai Matni$^{1}$
\thanks{$^*$ $\alpha$-$\beta$ equal contribution}
\thanks{$^1$ Electrical and Systems Engineering, University of Pennsylvania. Emails: \tt\small\{brucele, ttz2, nmatni\}@seas.upenn.edu}%
\thanks{$^{2}$Electrical Engineering, Columbia University. Emails: \tt\small\{lt2879, james.anderson\}@columbia.edu}%
}

\begin{document}

\twocolumn

\maketitle
\thispagestyle{empty}
\pagestyle{empty}

\begin{abstract}
    Representation learning is a powerful tool that enables learning over large multitudes of agents or domains by enforcing that all agents operate on a shared set of learned features. 
    However, many robotics or controls applications that would benefit from collaboration operate in settings with changing environments and goals, whereas most guarantees for representation learning are stated for static settings. Toward rigorously establishing the benefit of representation learning in dynamic settings, we analyze the regret of multi-task representation learning for linear-quadratic control. This setting introduces unique challenges. Firstly, we must account for and balance the \emph{misspecification} introduced by an approximate representation. Secondly, we cannot rely on the parameter update schemes of single-task online LQR, for which least-squares often suffices, and must devise a novel scheme to ensure sufficient improvement. We demonstrate that for settings where exploration is ``benign'', the regret of any agent after $T$ timesteps scales as $\tilde\calO(\sqrt{T/H})$, where $H$ is the number of agents. In settings with ``difficult'' exploration, the regret scales as $\tilde\calO(\sqrt{\du \dtheta} \sqrt{T} + T^{3/4}/H^{1/5})$, where $\dx$ is the state-space dimension, $\du$ is the input dimension, and $\dtheta$ is the task-specific parameter count. In both cases, by comparing to the minimax single-task regret $\calO(\sqrt{\dx\du^2}\sqrt{T})$, we see a benefit of a large number of agents. Notably, in the difficult exploration case, by sharing a representation across tasks, the effective task-specific parameter count can often be small $\dtheta < \dx\du$. Lastly, we provide numerical validation of the trends we predict.
\end{abstract}
\section{Introduction}

Many modern applications of robotics and controls involve simultaneous control over a large number of agents. For example, robot fleet learning, in which fleets of robots performing diverse tasks share information to learn more effectively, has demonstrated impressive success in recent years \citep{brohan2022rt, wang2023fleet}. One of the technologies that enables this success is \textit{transfer learning}, in which dynamics models or control policies built upon learned compressed features (also known as \textit{representation learning}) that are broadly useful for ensuing tasks of interest. Existing work which characterizes the generalization capabilities of transfer learning largely considers static environments, where data from an agent's completed task is aggregated with data from other agents to learn the shared features offline, rather than during task execution. However, it is often relevant to have a fleet of agents adapt quickly to a changing environment, e.g.\ a team of drones  flying in close proximity adapting to weather conditions, or a team of legged robots adapting to changing terrain conditions. In such settings, the agents must communicate to adjust their shared features online. 

In this work, we rigorously study such approaches for online fleet learning with dynamical systems in the analytically tractable setting of adaptive linear-quadratic (state-feedback) control. Adaptive linear-quadratic control has emerged as a benchmark for learning to control dynamical systems using online data. This consists of a learner interacting with an unknown linear system
\begin{align}
    \label{eq: single dynamics}
    x_{t+1} = A_\star x_t + B_\star u_t + w_t, \quad t \geq 1,
\end{align} 
with state $x_t$, input $u_t$, and noise $w_t$  assuming values in $\R^{\dx}$, $\R^{\du}$, and $\R^{\dx}$, respectively. The learner is evaluated by its incurred  \emph{regret}, which compares the cost incurred by playing the learner for $T$ time steps against the cost attained by the optimal LQR controller. Prior work typically studies regret of a single dynamical system of the form \eqref{eq: single dynamics}. In this work, we study a setting where there are $H \gg 1$ distinct systems which share an unknown $\dtheta$-dimensional dynamics basis. Each agent aims to minimize their individual linear-quadratic control objective; however, by communicating they may more efficiently learn the shared dynamics basis matrices. The broad questions we address are the following:
\begin{itemize}[leftmargin=*]
    \item What are the requisite algorithmic elements that enable simultaneous online control of \textit{multiple} systems?
    \item What are the concrete benefits of sharing a representation across agents compared with learning individual models for each agent?
\end{itemize}

\subsection{Related Work}

\noindent\textbf{Fleet Learning: }
Fleet Policy Learning considers a setting where a dataset is obtained from a diverse collection of robot interactions. It has been studied from the perspective of offline reinforcement learning  \citep{kumar2022pre} and for multi-task behavior cloning \citep{brohan2022rt, brohan2023rt, goyal2023rvt}. The centralized setup considered in this line of work is challenging to scale to many platforms. In particular, data communication and storage can become prohibitive, as can the training of the model.  Frameworks have also proposed and analyzed a weight merging approach where each platform learns a policy, and then communicates the weights to a central server that merges the weights \citep{wang2023fleet}. This work focuses on aggregating more skills by communicating, however the communication can also be used by multiple agents to adapt to a changing environment. This is the framework we analyze in this paper, where agents communicate their estimates for a set of shared parameters. This bears resemblance to certain federated or distributed learning settings with heterogeneous data, where due to privacy or compute constraints agents do not centralize raw data \citep{collins2021exploiting, ma2022state, tan2022fedproto}.


\noindent\textbf{Multi-Task Learning (in Dynamical Systems):} 
Multi-task learning has long been studied in machine learning \citep{baxter2000model}. More recently, multiple works have studied the benefit of a shared representation in iid learning with regard to generalization \citep{du2020few, tripuraneni2020theory} and efficient algorithms \citep{collins2021exploiting, vaswani2024efficient, thekumparampil2021sample, tripuraneni2021provable}. However, data generated from dynamical systems break key assumptions in these works.
With respect to dynamical systems, multiple works consider a parallel setting where all agents share a parameter space and task-specialization comes from perturbations therein, see Model-agnostic meta-learning (MAML) \citep{finn2017model}.\footnote{This is distinct from our setting, where agents share a \textit{representation} function and task-specialization comes from linear functions of the representation.}
Both model-free federated learning of the linear-quadratic regulator with data from heterogeneous systems \cite{wang2023model} and MAML for linear-quadratic control have been considered \cite{toso2024meta}.  However, both of these settings only recover optimality up to a heterogeneity bias. By instead imposing all dynamics matrices share a common basis \cite{modi2021joint}, one can ensure the error decreases to zero as data increases.
Analogous multi-task learning over dynamical systems settings have also been considered in imitation learning \citep{zhang2023multi, guo2023imitation}.  Most relevant to our work is \citet{zhang2024sampleefficient}, where the shortcomings of algorithms for iid representation learning are addressed for a related linear system-identification set-up. A component of our algorithm is adapted from their work.

\noindent\textbf{Regret Analysis of Adaptive Control: }
Our setting and analysis builds off recent work that attempts to provide finite sample guarantees for adaptive control by controlling the \emph{regret} of the learning algorithm. While adaptive control has a rich history beginning with autopilot development for high-performance aircraft in the 1950s \citep{gregory1959proceedings}, finite sample regret analysis  of adaptive control arose much later \cite{abbasi2011regret}. Subsequent work \citep{dean2018regret, cohen2019learning, mania2019certainty} has introduced algorithms that yield $\sqrt{T}$ regret, and are computationally feasible. \citet{simchowitz2020naive} establish corresponding lower bounds, indicating that a rate of $\sqrt{\du^2\dx T}$ is optimal for completely unknown systems. Improved regret bounds of $\texttt{poly}(\log T)$ are achievable when either $A^\star$ or $B^\star$ is known \citep{cassel2020logarithmic, jedra2022minimal}. The aforementioned work studies adaptive control in a setting where the noise is zero-mean and stochastic. Alternative formulations of the adaptive LQR problem consider bounded adversarial disturbances \citep{hazan2020nonstochastic, simchowitz2020improper} and settings where there is misspecification between the underlying data generating process and the model class \cite{ghai2022robust, lee2023nonasymptotic}. Our work extends analogous regret analysis to the multi-agent setting.

\subsection{Contribution}

We propose and analyze fleet linear-quadratic adaptive control in a setting where multiple linear systems driven by dynamics in the span of $\dtheta$ common basis matrices can communicate to drastically improve their individual control objectives. We propose such an algorithm and analyze the regret incurred, uncovering an interesting transition distinguishing the difficulty of the problem:
\begin{itemize}
    \item When the system specific parameters are ``benign'' to identify, our proposed  scheme incurs a regret of 
    \begin{align*}
        R_T = \tilde \calO\paren{ \sqrt{\frac{T}{H}}},
    \end{align*}
    where $H$ is the number of communicating agents. When there are many agents, this is drastically lower than the regret $\calO(\sqrt{\dx \du^2 T}$) incurred if each agent had to learn to control its respective system without communication.
    
    \item When the system-specific parameters are challenging to identify, our proposed algorithm incurs a regret of at most \begin{align*}
        R_T = \tilde \calO \paren{\sqrt{\du \dtheta }\sqrt{T} +  \frac{T^{3/4}}{H^{1/5 }}}.
    \end{align*}
    When $T$ is moderate, or if the number of agents $H$ is large, this can demonstrate a marked gain over the single-agent setting. However, when $T$ is large, the $T^{3/4}$ term dominates, which arises due to the mismatch between the difficulty of parameter identification and the misspecification of the learned basis directions.
    
\end{itemize}

In order to establish such guarantees, we propose and analyze a new algorithm that synthesizes tools from regret analysis of misspecified linear system identification and algorithmic analysis of multi-task linear regression. In particular, the multi-agent setting introduces unique challenges:
\begin{itemize}
    \item Due to the approximate representation at any given timestep, the problem is misspecified. Therefore, in addition to the standard explore-commit tradeoff, we must account for improving the representation.
    \item Whereas for prior work in the stochastic single-agent setting least-squares--whose optimization and generalization is well-understood--suffices algorithmically, such an analog is not well-posed for the multiple agent setting.
\end{itemize}
We validate our theory with numerical simulations, and demonstrate the value of communicating with similar agents to learn to control more efficiently. 

\noindent \textbf{Notation: } The Euclidean norm of a vector $x$ is denoted $\norm{x}$. For a matrix $A$, the spectral norm is denoted $\norm{A}$, and the Frobenius norm is denoted $\norm{A}_F$. 
The spectral radius of a square matrix is denoted $\rho(A)$. A symmetric, positive semi-definite (psd) matrix $A = A^\top$ is denoted $A \succeq 0$. 
The $\scurly{\min, \max}$ eigenvalue of a psd matrix $A$ is denoted $\scurly{\lambda_{\min}(A), \lambda_{\max}(A)}$. For a positive definite matrix $A$, we denote the condition number as $\kappa(A) \triangleq \frac{\lambda_{\max}(A)}{\lambda_{\min}(A)}$.
We denote the normal distribution with mean $\mu$ and covariance $\Sigma$ by  $\calN(\mu, \Sigma)$. 
For $f,g :D \to \mathbb{R}$, we write $f \lesssim g$ if for some $c > 0$, $f(x) \leq c g(x)\,\forall x \in D$. We denote the solutions to the discrete Lyapunov equation by $\dlyap(A,Q)$ and the discrete algebraic Riccati equation by $\dare(A,B,Q,R)$. 
For an integer $n\in \N$, we define the shorthand $[n] \triangleq\{1,\dots,n\}$.
Generally, we use $\scurly{\wedge, \vee}$ to denote a $\scurly{\min, \max}$ over an indicated quantity.

\section{Problem Formulation}

\subsection{System and Data assumptions} 

Consider $H$ systems with dynamics defined by 
\begin{align}
    \label{eq: dynamics}
    x_{t+1}^{(h)} = A_\star^{(h)} x_t^{(h)} + B_\star^{(h)} u_t^{(h)} + w_t^{(h)}, \quad t \geq 1,
\end{align}
for $h\in[H]$. We suppose that each rollout starts from initial state $x_1^{(h)}=0$ for $h \in [H]$, and that that the noise $w_t^{(h)}$ has iid elements that are mean zero and $\sigma^2$-sub-Gaussian for some $\sigma^2\in \R$ with $\sigma^2 \geq 1$ \citep{vershynin2018high}. We additionally assume that the noise has identity covariance: $\E \brac{w_t^{(h)}  w_t^{(h),\top}} = I$.\footnote{Noise that enters the process through a non-singular matrix $S$ can be addressed by rescaling the dynamics by $S^{-1}$.} 
We suppose the dynamics matrices admit the  decomposition
\begin{align}
    \label{eq: low dim structure}
    \bmat{A_\star^{(k)} & B_\star^{(k)}} &= \VEC^{-1}\paren{\Phi_\star \theta_\star^{(k)}}, 
\end{align}
where $\Phi_\star\in \R^{\dx(\dx+\du)\times \dtheta}$ is a column-orthonormal matrix that contains an optimal set of $\dtheta$ (vectorized) basis matrices in $\R^{\dx(\dx + \du)}$, and $\theta_\star^{(k)} \in\R^{\dtheta}$ are agent-specific parameters. The operator $\VEC^{-1}$ maps a vector in $\R^{\dx(\dx+\du)}$ into a matrix in $\R^{\dx \times (\dx+\du)}$ by stacking contiguous length-$\dx$ blocks of a vector (top-to-bottom) into columns of a matrix (left-to-right). We can equivalently write this as a linear combination of basis matrices:
\begin{align*}
    \bmat{A_\star^{(k)} & B_\star^{(k)}} = \sum_{i=1}^{\dtheta} \theta_{\star,i}^{(k)} \bmat{\Phi^{A}_{\star,i} & \Phi^{B}_{\star,i}},
\end{align*}
where $\bmat{\Phi^{A}_{\star,i} & \Phi^{B}_{\star,i}} = \VEC^{-1} \Phi_{\star,i}$
and $\Phi_{\star,i}$ is the $i^{\mathrm{th}}$ column of $\Phi_\star$. This decomposition of the data generating process is a natural extension of the low-rank linear representations considered in \cite{du2020few, zhang2023multi, zhang2024sampleefficient} to the setting of multiple related dynamical systems with shared structure determined by $\Phi_\star$. A version of this model for autonomous systems was considered by \cite{modi2021joint} for multi-task system identification.

\subsection{Control Objective}

The goal of the learners is to interact with system \eqref{eq: dynamics} while keeping the total cumulative cost small, where the system specific cumulative cost for system $h$ is defined for matrices $Q \succeq I$ and $R = I$ as\footnote{Generalizing to arbitrary $Q\succ 0$ and $R\succ 0$ can be performed by scaling the cost and changing the input basis.} 
\begin{align*}
    C_T^{(h)} \triangleq \sum_{t=1}^T c_t^{(h)}, \,\mbox{and}\, c_t^{(h)} \triangleq x_t^{(h),\top} Q x_t^{(h)}  + u_t^{(h),\top} R u_t^{(h)}.
\end{align*} 
To define an algorithm that keeps the cost small, we first introduce the infinite horizon LQR cost:
\begin{align}
    \label{eq: lqr cost}
    \calJ^{(h)}(K) \triangleq \limsup_{T\to\infty} \frac{1}{T} \E^K C_T^{(h)},
\end{align}
where the superscript $K$ denotes evaluation under the state-feedback controller $u_t^{(h)} = Kx_t^{(h)}$. To ensure that there exists a controller such that \eqref{eq: lqr cost} is finite, we assume $(A_\star^{(h)}, B_\star^{(h)})$ is stabilizable for all $h \in [H]$. Under this assumption, 
\eqref{eq: lqr cost} is minimized by the LQR controller $K_{\infty}(A_\star^{(h)},B_\star^{(h)})$, where
\begin{align*}
    K_{\infty}(A,B) &\triangleq - (B^\top P_{\infty}(A,B) B + R)^{-1} B^\top P_{\infty}(A,B) A,\\
    P_{\infty}(A, B) &\triangleq \dare(A,B,Q,R).
\end{align*}
We define the shorthands $P^{(h)}_\star \triangleq P_{\infty}(A_\star^{(h)}, B_\star^{(h)})$ and $K_\star^{(h)} \triangleq K_\infty(A_\star^{(h)}, B_\star^{(h)})$ for all $h \in [H]$. To characterize the infinite-horizon LQR cost of an arbitrary stabilizing controller $K$, we additionally define the solution $P_K^{(h)}$ to the Lyapunov equation for the closed loop system under an arbitrary $K$ where $\rho(A_\star^{(h)} + B_\star^{(h)} K) < 1$: 
\begin{align*}
    P_K^{(h)} \triangleq \dlyap(A_\star^{(h)} + B_\star^{(h)} K, Q+ K^\top R K).
\end{align*}
For a controller $K$ satisfying $\rho(A_\star^{(h)} + B_\star^{(h)} K) < 1$, $\calJ^{(h)}(K) = \trace(P_K^{(h)})$. We have that $P_{K_\star^{(h)}}^{(h)} = P_\star^{(h)}$. 

The infinite horizon LQR controller provides a baseline level of performance that our learner cannot surpass in the limit as $T \to \infty$. We quantify the performance of our learning algorithm by comparing the cumulative cost $C_T^{(h)}$ to the scaled infinite horizon cost attained by the LQR controller if the system matrices $\bmat{A_\star^{(h)} & B_\star^{(h)}}$ were known:
\begin{align}
    \label{eq: regret}
    \calR_T^{(h)} \triangleq C_T^{(h)} - T \calJ^{(h)}(K_\star^{(h)}).
\end{align}
This metric has previously been considered for adaptive control of a single system \citep{abbasi2011regret}. 
The above formulation casts the goal of the learner as interacting with each system \eqref{eq: dynamics} to maximize the information required for control while simultaneously regulating each system to minimize $\calR_T^{(h)}$. The learner uses its history of interaction with each system to do so by constructing dynamics models, e.g. by determining estimates $\hat A^{(h)}$
and $\hat B^{(h)}$. It may then use these estimates as part of a \emph{certainty equivalent} (CE) design by synthesizing controllers $\hat K^{(h)} = K_\infty(\hat A^{(h)}, \hat B^{(h)})$. It is known from prior work that if the model estimate is sufficiently close to the true dynamics, then the excess cost of playing the controller $\hat K^{(h)}$ is bounded by its parameter estimation error \citep{mania2019certainty, simchowitz2020naive}.

\begin{lemma}[Theorem 3 of \citep{simchowitz2020naive}]
    \label{lem: CE closeness main body} 
    \sloppy Define $\varepsilon^{(h)} \triangleq \tfrac{\snorm{P_\star^{(h)}}^{-10}}{3000} $.
    If 
    $
        \norm{\bmat{\hat A^{(h)} & \hat B^{(h)}}-\bmat{A_\star^{(h)} & B_\star^{(h)}}}_F^2 \leq \varepsilon^{(h)}, 
    $
    then 
    \begin{align*}
        &\calJ^{(h)}(\hat K^{(h)})  - \calJ^{(h)}(K_\star^{(h)}) \leq  \\
        &142\norm{P_\star^{(h)}}^8 \norm{\bmat{\hat A^{(h)} & \hat B^{(h)}}-\bmat{A_\star^{(h)} & B_\star^{(h)}}}_F^2.
    \end{align*}
\end{lemma}

\subsection{Algorithm Description}

\setlength{\textfloatsep}{0pt}
\begin{algorithm}
 \caption{Shared-Representation Certainty-Equivalent Control with Continual Exploration} 
 \label{alg: ce with exploration}
\begin{algorithmic}[t]
\State \textbf{Input: } Stabilizing controllers $K_0^{(h)}$ for $h \in [H]$, initial epoch length $\tau_1$, number of epochs $k_{\fin}$, exploration sequence $\sigma_1^2, \sigma_2^2, \sigma_3^2, \dots \sigma_{k_{\fin}}^2$, state bound $x_b$, controller bound $K_b$, initial representation estimate $\Phi_0$, gradient steps per epoch $N$
\State \textbf{Initialize: } $\hat K_1^{(h)} \gets K_0^{(h)}$,  $\tau_0 \gets 0$, $T \gets \tau_1 2^{k_{\fin}-1}$, $\hat\Phi_1 \gets \Phi_0$.
\For{$k=1,2, \dots, k_{\fin}$}
            \State \texttt{// Data collection}
            \For{$h=1,\dots, H$ \textbf{(in parallel)}} 
            \For{$t = \tau_{k-1}, \tau_{k-1} + 1, \dots, \tau_{k}$}
            
            \If{$\snorm{x_t^{(h)}}^2 \geq x_b^2 \log T$ or $\snorm{\hat K_k^{(h)}}\geq K_b$}\label{line: uniform bound}
                \State \textbf{Abort} and play $K_0^{(h)}$ forever
            \EndIf
            \State Play $u_t^{(h)} = \hat K_k^{(h)} x_t^{(h)} + \sigma_k g_t^{(h)}$, 
            
            \qquad \quad where $g_t^{(h)} \iidsim \calN(0,I)$ 
        \EndFor 
        \State \texttt{// Task-wise parameter updates}
        \State $\hat \theta_k^{(h)} \gets \texttt{LS}(\hat \Phi_k, x^{(h)}_{\tau_{k-1}:\lceil \frac{3}{2} \tau_{k-1} \rceil}, u_{\tau_{k-1}:\lceil\frac{3}{2}\tau_{k-1} \rceil}^{(h)})$ 
        \State $\bmat{\hat A_{k}^{(h)} & \hat B_{k}^{(h)}} \gets \VEC^{-1} \paren{\hat \Phi_k \hat \theta_{k}^{(h)}} $
        \State $\hat K_{k+1}^{(h)} \gets K_\infty(\hat A_{k}^{(h)}, \hat B_{k}^{(h)})$
        \EndFor
        \State \texttt{// Representation update}
        \State $\hat\Phi_{k+1} \gets \texttt{DFW}(\hat \Phi_{k}, x^{(1:H)}_{\lceil \frac{3}{2}\tau_{k-1}\rceil:\tau_k }, u_{\lceil\frac{3}{2} \tau_{k-1}\rceil:\tau_k}^{(1:H)}, N)$
    \State $\tau_{k+1} \gets 2 \tau_k$
\EndFor
\end{algorithmic}
\end{algorithm}
\setlength{\textfloatsep}{0pt}
\begin{algorithm}[t]
\caption{Least squares: $\texttt{LS}(\hat \Phi, x_{1:t+1}, u_{1:t})$} 
\label{alg: least squares}
\begin{algorithmic}[1]
\State \textbf{Input:} Model structure estimate $\hat \Phi$, state data $x_{1:t+1}$, input data $u_{1:t}$
\State \textbf{Return: } $\hat \theta$, where 
\begin{align*}
    \hat \theta &= \Lambda^\dagger \paren{\sum_{s=1}^{t} \hat \Phi^\top \paren{ \bmat{x_s \\ u_s} \otimes I_{\dx}} x_{s+1} } \quad \mbox{ and }\\
    \Lambda &= \sum_{s=1}^{t } \hat \Phi^\top \paren{ \bmat{x_s \\ u_s}\bmat{x_s \\ u_s} ^\top \otimes I_{\dx}}\hat \Phi.
\end{align*}
\end{algorithmic}
\end{algorithm}

\setlength{\textfloatsep}{0pt}
\begin{algorithm}[t]
\caption{\texttt{De-bias \& Feature Whiten}: \texttt{DFW}$(\hat \Phi, x_{1:t}^{(1:H)}, u_{1:t}^{(1:H)}, N)$}
\label{alg: dfw}
\begin{algorithmic}[1]
\State \textbf{Input:} Representation estimate $\hat \Phi$, state data $x_{1:t+1}^{(1:H)}$, input data $u_{1:t}^{(1:H)}$, gradient steps $N$, step-size $\eta$
\State Split each trajectory into subtrajectories of length $t_1$ and $t_2$, $N(t_1 + t_2) \leq t$.
\For{$n=1,\dots, N$}
\For{$h=1,\dots,H$ \textbf{in parallel}}
\State Compute weights 

\qquad $\hat\theta^{(h)}_n \gets \texttt{LS}(\hat \Phi_n, \scurly{x_s^{(h)}, u_s^{(h)}}_{s\in [t_1]})$.
\State Compute local rep.\ update $\overline\Phi_n^{(h)}$ \eqref{eq: dfw rep update} on $s \in [t_2]$.
\EndFor
\State Compute global rep.\ update 

\; $\hat \Phi^{n}, \_ \gets \texttt{thin\_QR}(\frac{1}{H} \sum_{h=1}^H \bar \Phi^{(h)}_{n})$. \label{line: DFW rep averaging}
\EndFor
\State \textbf{Return: }$\hat\Phi_+ \gets \hat \Phi_{N}$
\end{algorithmic}
\end{algorithm}

Our proposed algorithm, \Cref{alg: ce with exploration}, is a CE algorithm similar to those proposed by \citet{ cassel2020logarithmic, lee2023nonasymptotic}, which we extend to the multi-task representation learning setting. The algorithm takes a stabilizing controller $K_0^{(h)}$ for each system $h$ as an input, in addition to an initial epoch length $\tau_1$, an exploration sequence $\sigma_k^2$ for $k \in [k_{\fin}]$, state and controller bounds $x_b$ and $K_b$, an initial representation estimate $\Phi_0$, and a number of gradient steps $N$ to run on the representation per epoch. 
Starting from the initial controllers, \Cref{alg: ce with exploration} follows a doubling epoch approach. During each epoch, each agent plays their current controller with exploratory noise added with scale determined by the exploration sequence. Each agent then uses the collected data to estimate its dynamics $\bmat{\hat A^{(h)} & \hat B^{(h)}}$ by running least-squares (\Cref{alg: least squares}), fixing the current representation estimate $\hat\Phi$.\footnote{This procedure throws away data from previous epochs, and does not allow updating the model at arbitrary times. This eases the analysis, but may be undesirable. Such undesirable characteristics have been removed in single task expected regret analysis \citep{jedra2022minimal}.} This is used to synthesize a new CE controller $\hat K^{(h)} = K_\infty(\hat A^{(h)}, \hat B^{(h)})$. At the end of each epoch, the agents engage in a round of $N$ representation updates (\Cref{alg: dfw}), in which they update their estimate for the shared basis using local data and communicate to take the average of their estimates. To analyze expected regret it is necessary to prevent catastrophic failures even under unlikely failure events. For this reason, the algorithm checks the state and controller norm against the supplied bounds $x_b$ and $K_b$ at the start of each interaction round, and aborts the CE scheme if either is too large. 

A key subtlety and contribution of our algorithm comes in how the parameters are updated (\Cref{alg: ce with exploration} and \ref{alg: dfw}). In the single-agent setting, the optimal dynamics matrix $\bmat{\hat A & \hat B}$ with respect to the current data batch follows by least squares, such that with a doubling epoch the parameter error approximately halves \citep{simchowitz2020naive}.
However, due to the multi-agent structure of our setting, least squares is no longer implementable, let alone optimal. This motivates the need for an alternative subroutine that ensures a reduction in the representation error between epochs. Subroutines satisfying this are remote in the literature, especially since existing linear representation learning (or bilinear matrix sensing) algorithms heavily rely on the assumption that the data (or sensing matrix) across all tasks is iid isotropic Gaussian $x_i^{(h)} \iidsim \calN(0,I)$ \citep{collins2021exploiting, thekumparampil2021sample, tripuraneni2021provable}, which is violated in our setting where states  distribution from different systems converge to their respective stationary distributions. A recent algorithm  \texttt{De-bias \& Feature Whiten (DFW)} proposed by \citet{zhang2023meta} addresses many analogous issues for a related multi-task representation learning problem, which we adapt for our setting. 
Beyond its guarantees (see \Cref{sec: representation error guarantees}), \texttt{DFW} enables distributed optimization of a shared linear representation across data sources with \textit{non-identical distributions}, and temporally dependent covariates.
Additionally, \texttt{DFW} does not require communication of raw data between the agents, and instead each agent only communicates their respective updated representation, allowing the algorithm to be implemented in a federated manner.
During each \texttt{DFW} iteration $n\in [N]$, each agent uses a portion of its data to estimate its local parameters via least-squares given the current representation $\hat\Phi_{n-1}$ (see \Cref{alg: least squares}).
Then, each agent uses the other portion of its data to compute its \textit{local} representation descent step:
\begin{align}\label{eq: dfw rep update}
    \nabla^{(h)}_{\Phi, n} &\triangleq \nabla_\Phi \sum_{t \in D_{n}} \norm{x_{t+1}^{(h)} - \VEC^{-1}\paren{\hat\Phi \hat\theta^{(h)}_n}\bmat{x^{(h)}_t \\ u^{(h)}_t}}^2 \nonumber \\
    \hat\Sigma_n^{(h)} &\triangleq \sum_{t \in D_n} \paren{\bmat{x^{(h)}_t \\ u^{(h)}_t} \bmat{x^{(h)}_t \\ u^{(h)}_t}^\top  \otimes I_{\dx} } \\
    \overline \Phi_n^{(h)} &\gets \hat\Phi_{n-1} - \eta \;(\hat\Sigma_n^{(h)})^{-1} \nabla^{(h)}_{\Phi, n}. \nonumber
\end{align}
The updated local representations from each agent are then averaged and orthonormalized, and transmitted back to each agent for the next iteration (see \Cref{alg: dfw}, line \ref{line: DFW rep averaging}).

\subsection{Representation Error Guarantees}\label{sec: representation error guarantees}

In this section, we motivate the roles of our representation update (\Cref{alg: dfw}) and task-specific weight update (\Cref{alg: least squares}) subroutines. Consider current representation estimate $\hat\Phi$ and data $(x_{1:t}^{(1:H)}, u_{1:t}^{(1:H)})$ generated from initial states $x_1^{(1)}, \dots, x_1^{(H)}$, under stabilizing controllers $K^{(1)}, \dots, K^{(H)}$ with exploratory noise $\sigma_u g_s^{(h)}$, $g_s^{(h)} \iidsim \calN(0,I_{\du})$ for $s\in[t]$, $h\in[H]$, and some $\sigma_u \in [0,1]$. This can be seen as the general set-up for the data collected during an epoch of \Cref{alg: ce with exploration}. We want to establish the following:
\begin{enumerate}
    \item Running \texttt{DFW} yields an updated representation whose error decomposes as a contraction of the previous representation's error plus a variance term that scales inversely with the amount of \textit{total data} $tH$.

    \item The parameter error $\norm{\hat\Phi\hatthetah- \Phi_\star \thetahstar}$ accrued by fitting the least-squares task-specific weights, holding the representation fixed, decomposes into a sum of least-squares error scaling inversely with $t$ and the \textit{representation error}.
\end{enumerate}
These two guarantees together inform how to set the epoch length and exploratory noise strength $\sigma_u$ to balance the explore-commit tradeoff for the ensuing regret analysis. To quantify the representation error, we consider the \textit{subspace distance} between the spaces spanned by the columns of $\hat\Phi$ and $\Phi_\star$ (which are constrained to be column-orthonormal).
\begin{definition}[{\citet{stewart1990matrix}}]\label{def: subspace dist}
    For a given matrix with orthonormal columns $\Phi$, let $\Phi_{\perp}$ be a matrix such that $\bmat{\Phi & \Phi_{\perp}}$ is an orthogonal matrix. Then, given another column-orthonormal matrix $\Phi'$, the \emph{subspace distance} between $\Phi',\Phi$ may be written $d(\Phi,\Phi') \triangleq \|\Phi_\perp^\top \Phi'\|$.
\end{definition}
For all dimensions of $\Phi_\star$ to be identifiable, we also make the following full-rank assumption on the optimal weights $\theta^{(1)}_\star, \dots, \theta^{(H)}_\star$.
\begin{assumption}\label{assumption: full rank weights}
    Consider $\Phi_\star$, $\scurly{\thetahstar}$ such that $\VEC^{-1}(\Phi_\star \thetahstar) = \bmat{A^{(h)}_\star & B^{(h)}_\star}$, $h=1,\dots,H$. We assume $\rank\paren{\sumH \thetahstar \thetahstar^\top} = \dtheta$.
\end{assumption}
We now state a bound on the improvement of the subspace distance after running \Cref{alg: dfw}.
\begin{theorem}[\texttt{DFW} guarantee, redux]
    \label{thm: dfw bound informal}
    Let Assumption~\ref{assumption: full rank weights} hold and fix $\delta \in (0,1)$. Then, provided an appropriately chosen step-size $\eta > 0$, burn-in $t \geq \tau_{\mathsf{dfw}}$, and initial representation error $d(\hat\Phi, \Phi_\star) \leq d_{\mathsf{dfw}}$,
    with probability at least $1 - \delta$ running \Cref{alg: dfw} yields the following guarantee on the updated representation $\hat\Phi \to \hat\Phi_N$:
    \begin{align*}
        d(\hat \Phi_{N}, \Phi_\star) \leq \rho^N d(\hat \Phi, \Phi^\star) + \frac{\Kbaravg}{1 - \sqrt{2}\rho^N}\frac{\sqrt{N}}{\sigma_u \sqrt{t H}},
    \end{align*}
    where
    \begin{align*}
        \rho &= 1 - 0.897 \eta \lambda_{\min}\paren{\sumH \thetahstar \thetahstar^\top}\\
        \Kbaravg &= \sqrt{\avgsumH \sigma^2 \snorm{\thetahstar}^2 (2 + \snorm{K^{(h)}}^2)} \\
        &\quad\cdot \mathrm{poly}(\dx, \du, \log(H),  \log(1/\delta)).
    \end{align*}
\end{theorem}
In particular, we have demonstrated that running \texttt{DFW} contracts the subspace distance by a factor of $\rho^N$, up to a variance factor. Notably, $\Kbaravg$ serves as a task-averaged ``noise-level'', and the denominator of the variance factor scales \textit{jointly} with the number of tasks $H$ and data per task $t$. For downstream analysis, it suffices to choose a number of iterations $N$ such that $\rho^N \leq 1/2 $, i.e., $ N \geq \log(2)/\log(1/\rho)$, which is independent of the size of the data. The subspace distance manifests in the error between the learned system parameters $\hat\Phi \hat\theta$ and the optimal $\Phi_\star \theta_\star$.
In particular, given the output $\hat\theta$ of \Cref{alg: least squares}, 
it can be shown (e.g.\ Theorem 5, \cite{lee2023nonasymptotic}) that the parameter least squares error decomposes into a term scaling inversely with data and a term involving the subspace distance between $\hat \Phi$ and $\Phi_\star$. 
\begin{theorem}(LS error, informal)
    \label{thm: ls error informal}
    Consider running \Cref{alg: least squares} on the $t$ data samples generated from a system of the form \eqref{eq: dynamics} for $t \geq \tau_{\mathsf{ls}}$, where $\tau_{\mathsf{ls}}
$ is a burn-in time. Then with probability at least $1-\delta$, 
    \begin{align*}
        \norm{\hat \Phi \hat \theta - \Phi_\star \theta_\star}^2 \lesssim  \frac{\sigma^2 \dtheta \log(1/\delta)}{t \times \mathsf{excitation \, lvl}} + C_{\mathsf{sys}} \frac{d(\hat \Phi, \Phi_\star)^2}{\mathsf{excitation\,lvl}},
    \end{align*}
    where $C_{\mathsf{sys}}$ is a constant that depends on the system \eqref{eq: dynamics}, and $\mathsf{exictation\,lvl}$ characterizes the extent to which the the state is excited as required to identify the parameters $\theta$. 
\end{theorem}
Formal statements of \Cref{thm: dfw bound informal} and \Cref{thm: ls error informal} are instantiated in the ensuing regret analysis and can be found in the appendix. We have thus established the desiderata stated at the beginning of the section. It remains to show that salient choices of epoch length and exploratory noise level in \Cref{alg: ce with exploration} yield no-regret guarantees.


\section{Regret analysis}

As previewed in the introduction, we consider two settings: one where the system-specific parameters $\theta_\star^{(h)}$ are easily identifiable given the representation, and one in which they are not. The setting where the system-specific parameters are easily identifiable corresponds to a situation in which  $\mathsf{excitation\,lvl}$ from \Cref{thm: ls error informal} is nonzero even when the input is determined by the optimal LQR controller.
In both settings, we require that the bounds for the abort procedure (Line \ref{line: uniform bound}, \Cref{alg: ce with exploration}) are sufficiently large to ensure that the abort procedure occurs with small probability. To state the bounds, we introduce the following notation.
\begin{align*}
    \Psi_{B_\star^{(h)}} &\triangleq \max\curly{1,\norm{B_\star^{(h)}}}, \quad && \psibmax \triangleq \max_{h=1,\dots, H} \Psi_{B_\star^{(h)}}  \\ 
    \thetamax &\triangleq \max_{h=1,\dots, H} \norm{\theta_\star^{(h)}}, && \pzeromax \triangleq \max_{h=1,\dots, H} \norm{P_{K_0^{(h)}}^{(h)}} \\
    \pstarmin &\triangleq \min_{h=1,\dots, H} \norm{P_{K^{(h)}_\star}^{(h)}}, && \epsmin \triangleq \min_{h=1,\dots,H} \varepsilon^{(h)},
\end{align*}
where $\varepsilon^{(h)}$ is as in \Cref{lem: CE closeness main body}.

\begin{assumption}
    \label{asmp: state and controller bounds}
    We assume that 
    \begin{align*}x_b   &\geq 400 (\pzeromax)^2 \psibmax \sigma \sqrt{\dx+\du}, \quad K_b   \geq \sqrt{\pzeromax}.
        \end{align*}
\end{assumption} 

\subsection{Not Easily Identifiable}

In this setting, we do not make additional assumptions about the structure of $\Phi_\star$. We require an assumption ensuring that it is possible to obtain a stabilizing CE controller after the first epoch with high probability. To do so, we make an assumption about the size of the subspace distance of the representation estimate $\hat \Phi$ from $\Phi_\star$ after a single episode (leveraging the contraction of \Cref{thm: dfw bound informal}.) 
\begin{assumption}
    \label{asmp: upper bound on representation error exp}
    Define
    \begin{align*}
        \beta_1 &\triangleq C_{\beta,1} \sigma^4 (\pzeromax)^{12} (\psibmax)^8 (\thetamax)^2 (\dx+\du) \sqrt{\frac{\dtheta}{\du}}, \\
        \gamma_1 &\triangleq \frac{1}{C_{\gamma,1}} \frac{\sigma_1^2}{x_b^2 (P_0^\vee)^5 \Psi_{B^\star}^2 \sqrt{\kappa\paren{\sum_{h=1}^{\dtheta} \theta_\star^{(h)} \theta_{\star}^{(h),\top}}}}
    \end{align*}
    for   sufficiently large universal constants $C_{\beta,1}$ and $C_{\gamma,1}$. Let $\rho$ be as in \Cref{thm: dfw bound informal}. 
    We assume the initial subspace distance satisfies $ d(\Phi_0, \Phi_\star)\leq \min\curly{\frac{\epsmin}{4 H^{2/5} \beta_1}, \gamma_1}.$
\end{assumption}

This assumption leads to the following regret bound.

\begin{theorem}
    \label{thm: regret bound naive exploration}
     Consider applying \Cref{alg: ce with exploration} with initial stabilizing controllers $K_0^{(1)}, \dots K_0^{(H)}$ for $T = \tau_1 2^{k_{\fin}-1}$ timesteps for some positive integers $k_{\fin}$, and $\tau_1$. Let $\tau_k = 2^k \tau_1$ for $k\in[k_{\fin}]$.  
    Suppose that the exploration sequence supplied to the algorithm satisfies
          \begin{align}
          \label{eq: exploration, hard to learn}
          \sigma_k^2 = \max\curly{\tau_k^{-1/4} H^{-1/5},\; \sqrt{\frac{\dtheta}{\du\tau_k}}, \; \rho^{(k-1) N} d(\Phi_0, \Phi_\star)} \end{align}
          for $k\in[k_{\fin}]$, where $\rho$ is the contraction rate of \Cref{thm: dfw bound informal}. Suppose the state bound $x_b$ and the controller bound $K_b$ satisfy Assumption~\ref{asmp: state and controller bounds} and that $N \geq \log(2)/\log(1/\rho)$. Additionally suppose 
          that the weights satisfy Assumption~\ref{assumption: full rank weights}. 
    There exists a polynomial function $\texttt{poly}_{\mathsf{warm}}$ such that if $\tau_1 = \tau_{\mathsf{warm} } \log^9 T$ with 
    \begin{align*}
        \tau_{\mathsf{warm}}\geq \texttt{poly}_{\mathsf{warm}}\paren{\sigma, P_0^\vee, \Psi_B^\vee, \theta^\vee, x_b, d_\theta, \dx, \du, \log(H)},
    \end{align*}
    then the expected regret satisfies for $h=1,\dots, H$
     \begin{align*}
         \E \brac{ \calR_T^{(h)} }  \leq  c_0 \log^9(T) + c_1 \sqrt{\dtheta \du} \sqrt{T} \log^2 (T) \\
         + c_2 \frac{T^{3/4}}{H^{1/5}} \log^2 (HT),
     \end{align*}
   \sloppy where
   $
       c_0 =\texttt{poly}\bigg(\sigma, \dx, \du, \dtheta, x_b, K_b, \norm{Q}, \thetamax, \pzeromax, \psibmax,$ $ \tau_{\mathsf{warm}}, x_b, d(\hat \Phi_0,\Phi_\star), \log H \bigg),$ $
       c_1 = \texttt{poly}\paren{\pzeromax, \psibmax, \sigma}, $ and $
       c_2 = \texttt{poly}\bigg(\dx, \du, \dtheta, \pzeromax, \psibmax, \thetamax,\sigma, N\bigg).$
\end{theorem}

Note that $c_0$ and $c_2$ depend on various system and algorithm quantities, however $c_1$ depends only upon quantities which nominally do not depend on system dimension. This is to emphasize the dimension dependence of the $\sqrt{T}$ term in the regret bound. Consider the above bound in the regime where $T$ is small, e.g., on the order of the number of communicating agents. In this regime, the $T^{3/4}$ term becomes negligible, and the regret is dominated by the term that scales as $\sqrt{d_{\theta} \du} \sqrt{T}$. 
This should be contrasted with the minimax regret bound for single task adaptive control $\sqrt{\dx\du^2 T}$  \citep{simchowitz2020naive}: if the system-specific parameter count $d_{\theta}$ is smaller than $\dx\du$, then the dominant term in the low data regime is smaller than the minimax regret of the single-task setting. In the adaptive control setting under consideration, the low data regime is often the one of interest, as we want the controller to rapidly adapt to a changing environment. We note that the guarantees are not any time, as they require algorithm parameters to be chosen as a function of the time horizon $T$ (as required by the choice of $\tau_1 = \tau_{\mathsf{warm}} \log^2 T$ and the assumption that $N$ satisfies Assumption~\ref{asmp: upper bound on representation error exp}.)

It is remains an open question whether it is possible to achieve overall $\sqrt{T}$ regret in the multi-task learning setting. The following section examines one case where this is true.

\subsection{Easily identifiable}

In this setting, we assume that $\Phi^\star$ admits additional structure that makes the identification of $\theta_\star^{(h)}$ easy. 

\begin{assumption}
    \label{asmpt: persisent excitation}
     Let $\alpha$ be a number satisfying $\alpha \geq \frac{1}{3 \norm{P_\star^{\wedge}}^{3/2}}$. We assume that 
    $\lambda_{\min}\paren{\Phi_\star^\top \paren{\bmat{I \\ K} \bmat{I \\K}^\top \otimes I_{\dx}} \Phi_\star} \geq \alpha^2$ for $K = K_0^{(h)}, K_\star^{(h)}$ for $h\in [H]$. 
\end{assumption}

The above assumption captures a setting where playing either the initial controller $K_0$ or the optimal controller $K^\star$ provides persistence of excitation without any exploratory input if the representation were known. This can be seen by noting that the matrix $\Phi_\star^\top \paren{\bmat{I \\ K} \bmat{I \\K}^\top \otimes I_{\dx}} \Phi_\star$ is a lower bound (in Loewner order) for the covariance matrix formed by taking the expectation of $\Lambda/t$ in \Cref{alg: least squares} when $u_s = K x_s$ and $\hat \Phi = \Phi_\star$. 

Under the above assumption, the weights $\theta$ are easily identifiable once the shared structure $\Phi$ is learned. As in the previous section, we require that the initial representation error is small enough to guarantee the closeness condition in \Cref{lem: CE closeness main body} may be satisfied with our estimated model after a single epoch. 

\begin{assumption}
    \label{asmp: upper bound on representation error no exp}
    Define 
    \begin{align*}
        \beta_2 &\triangleq C_{\beta,2} \max_{h=1,\dots, H} \frac{\epsmin(\pzeromax)^9 (\psibmax)^8 (\thetamax)^2 (\dx+\du)}{\dtheta \min\curly{\alpha^2, \alpha^4}} \\
        \gamma_2 & \triangleq \frac{1}{C_{\gamma,2}} \frac{\alpha^2}{x_b^2 (P_0^\vee)^5 \Psi_{B^\star}^2 \sqrt{\kappa\paren{\sum_{h=1}^{\dtheta} \theta_\star^{(h)} \theta_{\star}^{(h),\top}}}}
    \end{align*}
    for sufficiently large universal constants $C_{\beta,2}$ and $C_{\gamma,2}$. We assume the initial subspace distance satisfies $ d(\Phi_0, \Phi_\star)\leq \min\curly{\frac{\epsmin}{2\beta_1}, \gamma_2}.$
\end{assumption}

This allows us to state the following regret bound.

\begin{theorem}
    \label{thm: regret bound no exploration}
     Consider applying \Cref{alg: ce with exploration} with initial stabilizing controller $K_0^{(1)}, \dots, K_0^{(H)}$ for $T = \tau_1 2^{k_{\fin}}$ timeteps for some positive integers $k_{\fin}$, and $\tau_1$. 
     Let $\tau_k = 2^k \tau_1$ for $k\in[k_{\fin}]$ and suppose the exploration sequence is 
     \begin{align} 
        \label{eq: exploration easy to learn}
        \sigma_k^2 = \max\curly{\tau_{k}^{-1/2} H^{-1/2}, \rho^{(k-1) N} d(\Phi_0, \Phi_\star) }, 
    \end{align} 
    for all $k \in [k_{\fin}]$, where $\rho$ is the contraction rate of \Cref{thm: dfw bound informal}.
    Suppose the state bound $x_b$ and the controller bound $K_b$ satisfy Assumption~\ref{asmp: state and controller bounds}, and that $\Phi_\star$ satisfies Assumption~\ref{asmpt: persisent excitation} and $\Phi_0$ satsisfies Assumption~\ref{asmp: upper bound on representation error no exp}. Additionally suppose that the parameter $N$ is sufficiently large that $\rho^N \leq \frac{1}{2}$ and that the weights satisfy Assumption~\ref{assumption: full rank weights}. 
    There exists a polynomial $\texttt{poly}_{\mathsf{warm}}$ such that if $\tau_1 = \tau_{\mathsf{warm} } \log^4 T$ with 
    \begin{align*}
        \tau_{\mathsf{warm}}\!\geq \!\texttt{poly}_{\mathsf{warm}}\paren{\!\sigma, P_0^\vee, \Psi_B^\vee, \theta^\vee, \!x_b, d_\theta, \dx, \du, \log(H), \!\frac{1}{\alpha}},
    \end{align*}
     then the expected regret satisfies for $h=1,\dots, H$ satisfies
     \begin{align*}
         \E \brac{\calR_T^{(h)}} \leq  c_1 \log^4(T) + c_2 \frac{\sqrt{T}}{\sqrt{H}} \log^2(TH),
     \end{align*}
    \sloppy where
    $
        c_1 = \texttt{poly}\bigg(\sigma, \dtheta, \du, \dx, \frac{1}{\alpha},  \psibmax, \pzeromax, x_b,$ $ K_b, \thetamax, \norm{Q}, \tau_{\mathsf{warm}}, d(\hat \Phi_0, \Phi_\star), \log H\bigg)$ and $
         c_2 = \texttt{poly}\bigg(\sigma, \dtheta, \du, \dx, \frac{1}{\alpha}, \psibmax,\pzeromax, x_b, N\bigg).$

\end{theorem}

Consider once more the setting when the amount of data is on the order of the number of communicating agents. Here, the regret is dominated by a $\log T$ term. In particular, by sharing the ``hard to learn'' information, the communicating agents significantly simplify their respective adaptive control problems. Even in the regime of large $T$, the above regret bound improves upon what is possible in the single task setting as long as the number of agents is sufficiently large.
\section{Numerical Validation}

We now present numerical results to illustrate and validate our bounds. In particular, we compare our proposed multi-task representation learning approach for the adaptive LQR design (Algorithm \ref{alg: ce with exploration}) over the setting where a single system attempts to learn its dynamics by using its local simulation data and computes a CE controller on top of the estimated model. To this end, our experimental setup considers $H$ dynamical systems, described by \eqref{eq: dynamics}, where the system matrices $(A^{(h)}_\star, B^{(h)}_\star)$ are obtained by linearizing (around the origin) and discretizing (with Euler's approach) multiple cartpole dynamics with equations:
\vspace{-0.2cm}
\begin{align}\label{eq:cartpole_dynamics}
&(M^{(h)} + m^{(h)}) \ddot{x} + m^{(h)}\ell^{(h)}(\ddot{\theta}\cos(\theta) - \dot{\theta}^2\sin(\theta)) = u,\notag\\
&m^{(h)}(\ddot{x}\cos(\theta) + \ell^{(h)}\ddot{\theta} - g\sin(\theta)) = 0,
\vspace{-1cm}
\end{align}
for all $h \in [H]$, where $c^{(h)}_{\text{p}} = (M^{(h)}, m^{(h)}, l^{(h)})$ denote the tuple of cartpole parameters. Such parameters represent the cart mass, pole mass, and pole length, respectively. We set the gravity $g = 1$ and perform the discretization of \eqref{eq:cartpole_dynamics} with step-size $0.25$. Following \cite{lee2023nonasymptotic}, we generate $H$ $(A^{(h)}_\star, B^{(h)}_\star)$, by first considering a set of \emph{nominal} cartpole parameters: $c^{(1)}_{\text{p}} = (0.4, 1.0, 1.0)$, $c^{(2)}_{\text{p}} = (1.6, 1.3, 0.3)$, $c^{(3)}_{\text{p}} = (1.3, 0.7, 0.65)$, $c^{(4)}_{\text{p}} = (0.2, 0.055, 1.36)$, and $c^{(5)}_{\text{p}} = (0.2, 0.47, 1.825)$.

We then perturb such parameters with a random scalar within the interval $(0,0.1)$ to generate different cartpole parameters $c^{(h)}_p$. With the system matrices $(A^{(h)}_\star, B^{(h)}_\star)$ in hands, for all $h \in [H]$, we generate the disturbance signal as $w^{(h)}_t \sim \calN(0, 0.01 I_{\dx})$ and set the step-size and number of iterations of Algorithm \ref{alg: dfw} as $\eta = 0.25$, and $N = 1000$. It is worth noting that step 2 of Algorithm \ref{alg: dfw} is considered for the simplicity of the theoretical analysis only, in our experiments we exploit the entire dataset for all \texttt{DFW} iterations. 

Figure \ref{fig:regret_easily_identifiable} depicts the expected regret of Algorithm \ref{alg: ce with exploration} as a function of the timesteps $T$ for a varying number of tasks $H$. Note that such expected regret is with respect to a nominal task $h = 1$. This figure shows the results for the easily identifiable setting, i.e., where Assumption \ref{asmpt: persisent excitation} is satisfied. The labeled ``fully-unknown" curve corresponds to the setting where a single system estimates its dynamics and computes its controller only using its own trajectory data. As predicted in our bounds (Theorem \ref{thm: regret bound no exploration}), by learning the representation in a multi-task setting and exploiting it to learn a more accurate model can provide a significant reduction in the expected regret when compared to the fully-unknown case. In particular, the regret incurred in the single-task setting is in the order of $\mathcal{O}(\sqrt{T})$, whereas the regret of Algorithm \ref{alg: ce with exploration} in the easily identifiable setting is dominated by $\mathcal{O}\left(\frac{\sqrt{T}}{\sqrt{H}}\right)$. Therefore, as the number of tasks $H$ increases, the regret of Algorithm \ref{alg: ce with exploration} decreases. This can be seen comparing the regret from $H = 25$ to $H = 100$--which both improve upon the regret in the fully-unknown setting. 

\begin{figure}
    \centering
    \includegraphics[width=0.5\textwidth]{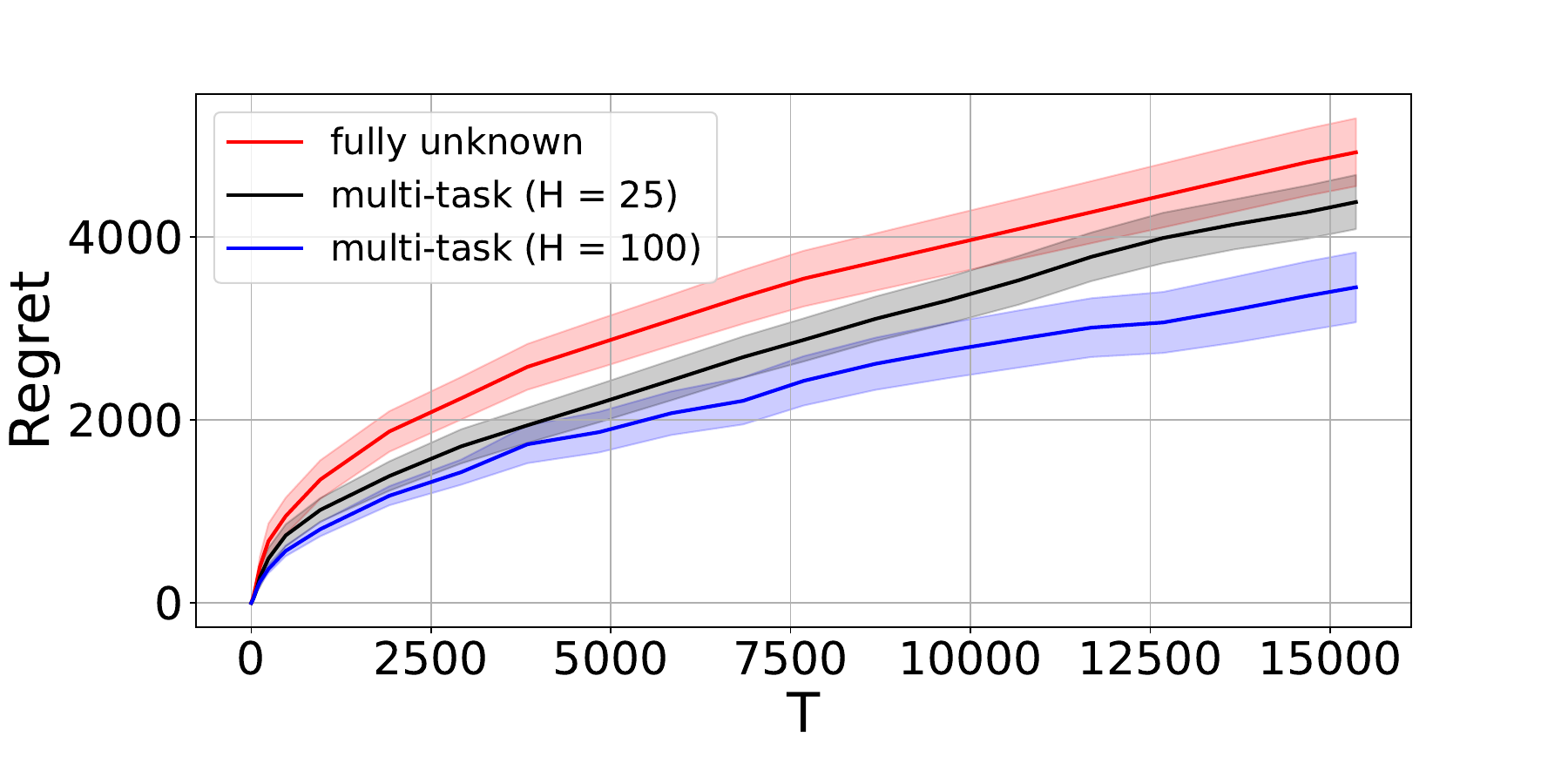}
    \vspace{-1cm}
    \caption{Regret of Algorithm \ref{alg: ce with exploration} with varying number of tasks $H$. We consider $k_{\fin} = 10$ epochs with initial epoch length $\tau_1 = 30$, an exploratory sequence scaling as $\sigma_k^2 \propto \frac{1}{\sqrt{2^k}}$, state and controller bounds $x_b = 25$, and $K_b = 15$, and random $\Phi_0$ with $d(\Phi_0, \Phi_\star) \approx 0.99$.} 
    \label{fig:regret_easily_identifiable}
\end{figure}
\section{Conclusion}
We proposed an algorithm for the simultaneous adaptive control of multiple linear dynamical systems sharing a representation. We leveraged recent results for representation learning with non-iid data in order to provide non-asymptotic regret bounds incurred by the algorithm in two settings: one where the system specific parameters are easily identified from the shared representation, and one where they are not. In the setting where the system specific parameters are easily identifiable, the regret scales as $\sqrt{T}/\sqrt{H}$, while in the difficult-to-identify setting, the regret scales as $T^{3/4}/H^1/5$. An interesting direction for future work is to determine whether the $T^{3/4}/H^{1/5}$ regret bound can be improved to $\sqrt{T}/\sqrt{H}$ even in the difficult-to-identify setting. It would also be interesting to extend the analysis of online adaptive control with shared representations to characterize the regret of learning to control certain classes of nonlinear systems, as has been done in the single task setting \cite{boffi2021regret}. 

\section*{Acknowledgements}

BL, TZ, and NM gratefully acknowledge support from NSF Award SLES 2331880, NSF CAREER award ECCS 2045834, and NSF EECS 2231349. LT is funded by the Columbia Presidential Fellowship. JA is partially funded by NSF grants ECCS 2144634 and 2231350 and the Columbia Data Science Institute.

\bibliographystyle{IEEEtranN}
\bibliography{refs, refs_dfw}

\onecolumn


\section{Outline for Proofs of  \Cref{thm: regret bound naive exploration} and \Cref{thm: regret bound no exploration}}

Our main results proceed by first defining a success events for which the certainty equivalent control scheme never aborts, and generates dynamics estimates $\bmat{\hat A_k^{(h)} & \hat B_k^{(h)}}$ which are sufficiently close to the true dynamics $\bmat{A_\star^{(h)} & B_\star^{(h)}}$ at all times. The success events 
are $\calE_{\mathsf{success},1} =  \calE_{\mathsf{bound}} \cap  \calE_{\mathsf{est},1} \cap \calE_{\mathsf{cont}}$ and $\calE_{\mathsf{success},2}= \calE_{\mathsf{bound}} \cap \calE_{\mathsf{est},2} \cap \calE_{\mathsf{cont}}$ for the settings where the task specific parameters are not easily identifiable and where they are, respectively. Here, 
\begin{align*}
    \calE_{\mathsf{bound}} &= \curly{\norm{x_t^{(h)}}^2  \leq x_b^2 \log T \quad \forall t\in[T], \, \forall h \in[H]} \cap \curly{\norm{\hat K_k^{(h)}} \leq K_b \,, \forall k \in [k_{\fin}],\, \forall h \in[H]}, \\
    \calE_{\mathsf{est},1} &= \curly{\norm{\bmat{\hat A_k^{(h)} & \hat B_k^{(h)}} - \bmat{A_\star^{(h)} & B_\star^{(h)}}}_F^2 \leq C_{\mathsf{est},1} \frac{\sigma^2 \dtheta \norm{P_{K_0}^{(h)}}}{\tau_{k} \sigma_k^2} \log (HT) + \frac{\beta_1 H^{2/5 } d(\hat \Phi_k, \Phi_\star)^2 }{\sigma_k^2}  \, \forall k\in[k_{\fin}], \, \forall h \in[H]}, \\
     \calE_{\mathsf{est},2} &= \curly{\norm{\bmat{\hat A_k^{(h)} & \hat B_k^{(h)}} - \bmat{A_\star^{(h)} & B_\star^{(h)}}}_F^2 \leq C_{\mathsf{est},2} \frac{\sigma^2  \dtheta   }{ \tau_{k} \alpha^2 }\log(HT)+ \beta_2 d(\hat \Phi_k, \Phi_\star)^2 \, \forall k \in [k_{\fin}], \, \forall h \in[H]}, \\
      \calE_{\mathsf{cont}} &= \curly{d(\hat \Phi_{k}, \Phi_\star) \leq \rho^{k N} d(\hat \Phi_{0}, \Phi_\star) + \frac{\Ccontract  \sqrt{N} \log (HT)}{(1-\sqrt{2}\rho^N) \sqrt{H \tau_{k} \sigma_k^2} } \, \forall k \in [k_{\fin}]},
\end{align*}
and $C_{\mathsf{est},1}$ and $C_{\mathsf{est},2}$ are positive universal constants. We recall that
\begin{itemize}
    \item $x_b$ and $K_b$ are the state and controller bounds triggering the abort procedure, see Assumption~\ref{asmp: state and controller bounds}.
    \item $\beta_1$ and $\beta_2$ are system theoretic constants defined in Assumption~\ref{asmp: upper bound on representation error exp} and Assumption~\ref{asmp: upper bound on representation error no exp}.
    \item $k_{\fin}$ is the total number of epochs run in \Cref{alg: ce with exploration}, and $\tau_k$ is the length of epoch $k$. 
    \item $\alpha$ is the parameter defined in Assumption~\ref{asmpt: persisent excitation} that quantifies the degree to which the initial and optimal controllers provide persistent excitation of the system specific parameters. 
    \item $\sigma_k^2$ is the level of input exploration during epoch $k$.
    \item $N$ is the number of descent steps run on the shared representation per epoch in \Cref{alg: dfw}.
     \item $\rho$ describes the radius of contraction for each iteration of \Cref{alg: dfw}, while $\Ccontract$ characterizes the numerator of the variance for each iteration; $\rho$ is defined in \Cref{thm: dfw bound informal} and $\Ccontract$ in \Cref{thm: dfw bound formal}.
\end{itemize}

With these events defined, the proofs for \Cref{thm: regret bound naive exploration} and \Cref{thm: regret bound no exploration} consist of two steps:
\begin{enumerate}
    \item In \Cref{s: success event probability bounds} we show that the success events $\calE_{\mathsf{success},1}$ and $\calE_{\mathsf{success},2}$ hold with high probability. 
    \item In \Cref{s: synthesizing the bounds}, we decompose the expected regret into a component incurred under the success event and under the failure event. We show that the regret incurred under the failure event is small. The regret under the success event then dominates the overall regret, which is in turn bounded to obtain the expressions in \Cref{thm: regret bound naive exploration} and \Cref{thm: regret bound no exploration}. 
\end{enumerate} 
Before doing so, we present formal versions of \Cref{thm: dfw bound informal} and \Cref{thm: ls error informal} in \Cref{s: technical preliminaries}.  
\section{Technical Preliminaries}
\label{s: technical preliminaries}

To bound the probability of failure, we require two key components for our analysis: a high probability bound on the estimation error in terms of the level of misspecificiation, and a bound showing that the contraction event holds with high probability for any one epoch. The bound for the first step is provided in \cite{lee2023nonasymptotic}, and the bound on the second step is provided in \cite{zhang2023meta}. We first describe the process characterizing the data collected during each epoch.

Consider a general estimation problem in which the system is excited by an arbitrary stabilizing controller $K$ and excitation level defined by $\sigma_u$.  In particular, we consider the evolution of the following system: 
\begin{equation}
\label{eq: rollout under K}
\begin{aligned}
    x_{t+1} &= A_\star x_t + B_\star u_t +  w_t \\
    u_t &= K x_t + \sigma_u g_t,
\end{aligned}
\end{equation}
where $g_t \overset{i.i.d.}{\sim} \calN(0,I)$, and $w_t$ is a random variable with $\sigma^2$-sub-Gaussian entries satifying $\E[w_t w_t^\top] =I  $. We assume that $\sigma_u^2\leq 1$ and that $x_1$ is a random variable. 

\subsection{Least squares error}
We first consider generating the estimates $\hat \theta, \Lambda = \texttt{LS}(\hat \Phi, x_{1:t+1}, u_{1:t})$. 
We present a bound on the estimation error $\norm{\hat \Phi \hat \theta - \Phi_\star \theta^\star}^2$ in terms of the true system parameters as well as the amount of data, $t$.  

\begin{theorem}[Misspecified LS Est. Error - Formal Version of \Cref{thm: ls error informal},  Theorem 5 of \citep{lee2023nonasymptotic}]
    \label{thm: ls estimation error}
    Let $\delta\in(0,1/2)$. Suppose $t \geq c \tau_{\mathsf{ls}}(K, \norm{x_1}^2, \delta)$ for 
      \begin{align*}
      \tau_{\mathsf{ls}}(K, \bar x, \delta) & \triangleq  \max\bigg\{\sigma^4  \norm{P_K}^3 \Psi_{B^\star}^2 \paren{\dx+\du + \log \frac{1}{\delta}},  \bar x \times \norm{P_K} + 1\bigg\}
      \end{align*}
      and a sufficiently large universal constant $c>0$. There exists an event $\calE_{\mathsf{ls}}$ which holds with probability at least $1-\delta$ under which the estimation error satisfies  
    \begin{align*}
        &\norm{\hat \Phi \hat \theta  - \Phi_\star \theta^\star}^2 \lesssim \frac{\dtheta \sigma^2  }{t \lambda_{\min} \paren{\bar \Delta^t(\sigma_u, K)}} \log\paren{ \frac{1}{ \delta} } + \paren{1+\frac{\sigma^4 \norm{P_K}^7 \Psi_{B^\star}^6 \paren{\dx+\du + \log \frac{1}{\delta}}}{t\lambda_{\min}(\bar \Delta^t(\sigma_u, K))^2 } } \frac{\norm{P_K}^2 \Psi_{B^\star}^2 d(\hat \Phi, \Phi_\star)^2 \norm{\theta^\star}^2}{\lambda_{\min}(\bar \Delta^t(\sigma_u, K))} .
    \end{align*}
    where
    \begin{align*}
        \bar \Delta^t(\sigma_u, K) \triangleq \hat \Phi^\top \brac{\paren{\frac{1}{t}\sum_{s=0}^{t-2} \sum_{j=0}^s \bmat{I \\ K} A_{K}^j (\sigma_u^2 B^\star \paren{B^\star}^\top + I)   \paren{A_K^j}^\top \bmat{I \\  K}^\top + \bmat{0 & 0 \\ 0 & \sigma_u^2 I_{\du}}} \otimes I_{\dx}} \hat \Phi.
    \end{align*}
\end{theorem}

\subsection{Representation Learning Guarantees from \texttt{DFW}}

We now want to prove that applying \texttt{DFW} leads to the high-probability contraction guarantee previewed in  Theorem \ref{thm: dfw bound informal}. Analogous to the original analysis provided in \cite{zhang2023meta}, to this end, we consider a general realizable regression setting, i.e.\ for each task $h$, the labels are generated by a ground truth mechanism
\begin{align*}
    \yhi = \VEC^{-1}(\Phi_\star \thetahstar) \xhi + \whi,
\end{align*}
where $\yhi \in \R^{\dy}$, $\xhi \in \R^{\dx}$. Note that our sysID setting simply follows by setting $\yhi \leftarrow x_{i+1}^{(h)}$, $\xhi \leftarrow \bmat{\xhi^\top & \mathopen{}u_i^{(h)}\mathclose{}^\top}^\top$. We assume for simplicity in this section that $(\xhi, \yhi)$ are iid sampled for $i=1,\dots,T$, $\xhi$ is $B^2$-subgaussian for all $i,h$ (which always holds due to the truncation step in \Cref{alg: ce with exploration}), and that the noise $\whi$ is $\sigmahw^2$-subgaussian for all $i$, where we may then instantiate to linear systems via a mixing-time argument as in \cite{zhang2023meta}.

For a given task $h$ and current representation $\hat\Phi$ and task-specific weights $\hatthetah$, the representation gradient with respect to a given batch of data $\scurly{(\xhi, \yhi)}_{i=1}^T$ can be expressed as
\begin{align*}
    \nabla_\Phi^{(h)} \triangleq\;&\nabla_\Phi \frac{1}{2T} \sum_{i=1}^T \norm{\yhi - \VEC^{-1}(\hat\Phi\hatthetah) \xhi }_2^2 \\
    =\;& \nabla_\Phi \frac{1}{2T} \sum_{i=1}^T \norm{\yhi - \paren{\xhi^\top \otimes I_{\dy}} \hat\Phi \hatthetah }_2^2 \\
    =\;& \frac{1}{T} \sum_{i=1}^T \paren{ \Xhi^\top \Xhi \hat\Phi \hatthetah \hatthetah^\top - \Xhi^\top \yhi \hatthetah^\top } \\
    =\;& \frac{1}{T} \sum_{i=1}^T \Xhi^\top \Xhi \paren{\hat\Phi \hatthetah - \Phi_\star \thetahstar} \hatthetah^\top - \frac{1}{T} \sum_{i=1}^T  \Xhi^\top \whi \hatthetah^\top,
\end{align*}
where $\Xhi \triangleq \xhi^\top \otimes I_{\dy}$
\citep{petersen2008matrix}. Recalling the definition of the orthogonal complement matrix $\Phi_{\star,\perp}$ and the subspace distance (\Cref{def: subspace dist}), we note that $\norm{\Phi_{\star,\perp}^\top \hat\Phi}_2$ is the subspace distance between $\Phi_\star$ and $\hat \Phi$, and $\Phi_{\star,\perp}^\top \Phi_\star = 0$. Noting these identities, the key insight in \texttt{DFW} \citep{zhang2023meta} is to pre-condition the representation gradient $\nabla_\Phi^{(h)}$ by the inverse sample-covariance $\paren{\frac{1}{T}\sum \Xhi^\top \Xhi}^{-1}$,
\begin{equation}
\begin{split}
    \tilde\nabla_\Phi^{(h)} &\triangleq \paren{\frac{1}{T}\sum \Xhi^\top \Xhi}^{-1} \nabla_\Phi^{(h)} \\
    &= \paren{\hat\Phi \hatthetah - \Phi_\star \thetahstar} \hatthetah^\top -  \paren{\sum_{i=1}^T \Xhi^\top \Xhi}^{-1}\sum_{i=1}^T 
    \Xhi^\top \whi \hatthetah^\top.
\end{split}
\end{equation}
Therefore, performing a descent step with the adjusted gradient $\tilde\nabla_\Phi^{(h)}$ and averaging the resulting updated representations across tasks $h$ yields
\begin{align}
    \overline\Phi_+ &= \frac{1}{H}\sum_{h=1}^H \paren{\hat\Phi - \eta \tilde\nabla_\Phi^{(h)}} \notag \\
    &= \hat\Phi \paren{I - \frac{\eta}{H}\sum_{h=1}^H \hatthetah \hatthetah^\top
    } + \Phi_\star  \paren{\frac{\eta}{H}\sum_{h=1}^H  \thetahstar \hatthetah^\top} + \frac{\eta}{H}\sum_{h=1}^H \paren{\sum_{i=1}^T \Xhi^\top \Xhi}^{-1}\sum_{i=1}^T 
    \Xhi^\top \whi \hatthetah^\top.
\end{align}
Pulling out the orthonormalization factor $R$ (via e.g.\ a QR decomposition) and left-multiplying the above by $\Phi_{\star, \perp}^\top$ yields
\begin{align*}
    \Phiperp^\top \hat\Phi_+ R &= \Phiperp^\top \hat\Phi \paren{I - \frac{\eta}{H}\sum_{h=1}^H \hatthetah \hatthetah^\top
    } + \Phiperp \frac{\eta}{H}\sum_{h=1}^H \paren{\sum_{i=1}^T \Xhi^\top \Xhi}^{-1}\sum_{i=1}^T 
    \Xhi^\top \whi \hatthetah^\top.
\end{align*}
Thus, if we establish orthonormalization factor $R$ is sufficiently well-conditioned, by taking the spectral norm on both sides of the above, we get the following decomposition
\begin{align}\label{eq: DFW contraction decomp}
    d(\hat\Phi_+, \Phi_\star) &\leq d(\hat\Phi, \Phi_\star) \norm{I - \frac{\eta}{H}\sum_{h=1}^H \hatthetah \hatthetah^\top} \norm{R^{-1}} + \norm{\frac{\eta}{H}\sum_{h=1}^H (\Xh \Xh^\top)^{-1} \Xh \Wh^\top \hatthetah^\top} \norm{R^{-1}}
\end{align}
As proposed in \cite{zhang2023meta}, bounding the improvement from $\hat\Phi$ to $\hat\Phi_+$ essentially reduces to establishing that $\norm{I - \frac{\eta}{H}\sum_{h=1}^H \hatthetah \hatthetah^\top}$ is a contraction with high-probability and analyzing the noise term $\frac{1}{H}\sum_{h=1}^H (\Xh \Xh^\top)^{-1} \Xh \Wh^\top$ as an \textit{average of self-normalized martingales}. The following adapts the analysis from \cite{zhang2023meta}, making necessary alterations due to the difference in definition of the representation.

Before proceeding, we note that \texttt{DFW} computes the least-squares weights $\hatthetah$ and representation gradient $\hat\Phi_+$ on disjoint partitions of data. We henceforth denote the subset of data used for computing the weights $\hatthetah$ by $\calT_1$, $\abs{\calT_1} = T_1$, and for $\hat \Phi_+$ by $\calT_2$, $\abs{\calT_2} = T_2$.
\,\\\\
\textbf{Contraction Factor}

As aforementioned, bounding the ``contraction rate'' of the representation toward optimality amounts to bounding
\begin{align*}
    \norm{I - \frac{\eta}{H}\sum_{h=1}^H \hatthetah \hatthetah^\top} \norm{R^{-1}}.
\end{align*}
We recall that $\hatthetah$ are the least-squares weights holding $\hat\Phi$ fixed:
\begin{align*}
    \hatthetah &\triangleq \argmin_{\theta} \sum_{i=1}^{T_1} \norm{\yhi - \VEC^{-1}(\hat\Phi \theta_\star) \xhi }_2^2 \\
    &= \paren{\sum_{i=1}^{T_1} \hat\Phi^\top \Xhi^\top \Xhi \hat\Phi }^{-1} \paren{\sum_{i=1}^{T_1} \hat\Phi^\top \Xhi^\top \yhi } \\
    &= \underbrace{\paren{\sum_{i=1}^{T_1} \hat\Phi^\top \Xhi^\top \Xhi \hat\Phi }^{-1} \paren{\sum_{i=1}^{T_1} \hat\Phi^\top \Xhi^\top \VEC^{-1}\paren{\Phi_\star \theta_\star}\xhi }  }_{(a)} + \underbrace{\paren{\sum_{i=1}^{T_1} \hat\Phi^\top \Xhi^\top \Xhi \hat\Phi }^{-1} \paren{\sum_{i=1}^{T_1} \hat\Phi^\top \Xhi^\top \whi }}_{(b)}.
\end{align*}
We recall the following properties of the Kronecker product:
\begin{align}
\begin{split}\label{eq: kronecker prod properties}
    (A \otimes B)(C \otimes D) &= (AC) \otimes (BD) \\
    \VEC(ABC) &= (C^\top \otimes A) \VEC(B).
\end{split}
\end{align}
Furthermore, given $\scurly{\lambda_i}, \scurly{\mu_j}$ the eigen/singular values of $A,B$, we recall the eigen/singular values of $A \times B$ are $\lambda_i \mu_j$ for all combinations $(i,j)$. With these facts in mind, we observe that term (a) above can be written as
\begin{align*}
    &\paren{\sum_{i=1}^{T_1} \hat\Phi^\top \Xhi^\top \Xhi \hat\Phi }^{-1} \paren{\sum_{i=1}^{T_1} \hat\Phi^\top \Xhi^\top \VEC^{-1}\paren{\Phi_\star \theta_\star}\xhi } \\
    =\;& \paren{\sum_{i=1}^{T_1} \hat\Phi^\top \Xhi^\top \Xhi \hat\Phi }^{-1} \paren{\sum_{i=1}^{T_1} \hat\Phi^\top \Xhi^\top \VEC^{-1}\paren{\Phi_\star \theta_\star}\xhi } \\
    =\;&  \paren{\sum_{i=1}^{T_1} \hat\Phi^\top \Xhi^\top \Xhi \hat\Phi }^{-1} \paren{\sum_{i=1}^{T_1} \hat\Phi^\top \paren{\xhi^\top  \otimes \Xhi^\top}\Phi_\star \theta_\star } \\
    =\;& \bigg(\hat\Phi^\top \bigg(\underbrace{\frac{1}{T_1}\sum_{i=1}^{T_1} \xhi \xhi^\top \otimes I_{\dy}}_{\triangleq \hatSigmaXh}\bigg) \hat\Phi \bigg)^{-1} \paren{\hat\Phi^\top \paren{\frac{1}{T_1}\sum_{i=1}^{T_1} \xhi \xhi^\top \otimes I_{\dy} } \Phi_\star \theta_\star } \\
    =\;& \paren{ \hat\Phi^\top \hatSigmaXh \hat\Phi }^{-1} \paren{\hat\Phi^\top \hatSigmaXh (\hat\Phi \hat\Phi^\top + \hatPhiperp\hatPhiperp^\top) \Phi_\star \theta_\star } \\
    =\;& \hat \Phi^\top \Phi_\star \theta_\star + \paren{ \hat\Phi^\top \hatSigmaXh \hat\Phi }^{-1} \paren{\hat\Phi^\top \hatSigmaXh \hatPhiperp\hatPhiperp^\top\Phi_\star \theta_\star }.
\end{align*}
We make particular note that in this form, $\norm{\hatPhiperp^\top\Phi_\star}_2 \triangleq d(\hat\Phi, \Phi_\star)$, directly reflecting the impact of misspecification on the least-squares weights.

Now recalling term (b) from above, we have
\begin{align*}
    \paren{\sum_{i=1}^{T_1} \hat\Phi^\top \Xhi^\top \Xhi \hat\Phi }^{-1} \paren{\sum_{i=1}^{T_1} \hat\Phi^\top \Xhi^\top \whi }.
\end{align*}
By viewing $\Xhi \hat\Phi$ as an $\R^{\dy \times \dtheta}$ matrix-valued stochastic process and $\whi$ as a $\sigma^2$-subgaussian $\R^{\dy}$-valued stochastic process that is independent of $\Xhi$, we may adapt \citet[Theorem 7]{lee2023nonasymptotic}.
\begin{lemma}[Adapted from {\citet[Theorem 7]{lee2023nonasymptotic}}]\label{lem: rep-conditioned SNM bound}
    Let $\Sigma \in \R^{\dtheta \times \dtheta}$ be a fixed positive-definite matrix. Then, with probability at least $1 - \delta$, we have
    \begin{align*}
        \norm{\paren{\Sigma +  \sum_{i=1}^{T_1} \hat\Phi^\top \Xhi^\top \Xhi \hat\Phi }^{-1/2} \paren{\sum_{i=1}^{T_1} \hat\Phi^\top \Xhi^\top \whi }}^2_2 &\leq \sigmahw^2 \log\paren{\frac{\det(\Sigma + \sum_{i=1}^{T_1} \hat\Phi^\top \Xhi^\top \Xhi \hat\Phi)}{\det(\Sigma)}} + 2\sigmahw^2 \log(1/\delta).
    \end{align*}
\end{lemma}
To instantiate the noise term bound, we combine \Cref{lem: rep-conditioned SNM bound} with a covariance concentration bound.
\begin{lemma}{\label{lem: rep-conditioned cov conc bound}}
    Define the population covariance matrix $\SigmaXh \triangleq \Ex\brac{\hatSigmaXh}$. Then, given $T_1 \gtrsim B^4 (\min\{\dtheta, \dx\} + \log(1/\delta))$, we have with probability at least $1 - \delta$
    \begin{align*}
        0.9 \hat\Phi^\top \SigmaXh \hat\Phi \preceq \hat\Phi^\top \hatSigmaXh \hat\Phi \preceq 1.1 \hat\Phi^\top \SigmaXh \hat\Phi.
    \end{align*}
    Therefore, setting $\Sigma = \SigmaXh$ in \Cref{lem: rep-conditioned SNM bound}, given $T_1 \gtrsim B^4 (\min\{\dtheta, \dx\} + \log(1/\delta))$, with probability at least $1 - \delta$, we have
    \begin{align*}
        &\norm{\paren{\sum_{i=1}^{T_1} \hat\Phi^\top \Xhi^\top \Xhi \hat\Phi }^{-1} \paren{\sum_{i=1}^{T_1} \hat\Phi^\top \Xhi^\top \whi }}_2^2
        \lesssim \sigmahw^2 \frac{\dtheta + \log(1/\delta)}{T_1 \lambda_{\min}(\hat\Phi^\top \SigmaXh \hat\Phi)}.
    \end{align*}
\end{lemma}
\textit{Proof of \Cref{lem: rep-conditioned cov conc bound}:} the first statement is an instantiation of a standard subgaussian covariance concentration bound, using the fact that $\xhi$ are $B^2$-subgaussian random vectors for all $i \in [T_1]$, see e.g.\ \citet[Lemma A.2]{zhang2023meta}. Since $\sum_{i=1}^{T_1} \Xhi^\top \Xhi = \paren{\sum_{i=1}^{T_1}\xhi \xhi^\top} \otimes I_{\dy}$, we have that
\begin{align*}
    c \Ex\brac{\sum_{i=1}^{T_1}\xhi \xhi^\top} \preceq \sum_{i=1}^{T_1}\xhi \xhi^\top \preceq C \Ex\brac{\sum_{i=1}^{T_1}\xhi \xhi^\top} \implies c \SigmaXh \preceq \hatSigmaXh \preceq C \SigmaXh,
\end{align*}
where the former event is precisely the covariance concentration event of the covariates $\xhi$. Therefore, if $T_1 \gtrsim B^4 (\dx + \log(1/\delta))$, then we have $0.9 \SigmaXh \preceq \hatSigmaXh \preceq 1.1 \SigmaXh$, and psd-ordering is preserved under pre and post-multiplying by any matrix $M^\top, M$, in particular setting $M = \hat\Phi$. Furthermore, if $\dtheta \leq \dx$, then by a standard argument (see e.g.\ \citet[Claim A.2]{du2020few}), then we have directly $T_1 \gtrsim B^4 (\dtheta + \log(1/\delta))$, $0.9 \hat\Phi^\top \SigmaXh \hat\Phi \preceq \hat\Phi^\top \hatSigmaXh \hat\Phi \preceq 1.1 \hat\Phi^\top \SigmaXh \hat\Phi$.

The latter statement follows by observing that conditioning on the covariance concentration event, we have
\begin{align*}
    \norm{\paren{\sum_{i=1}^{T_1} \hat\Phi^\top \Xhi^\top \Xhi \hat\Phi }^{-1/2} }^2 &\leq \paren{0.9 T_1 \lambda_{\min}\paren{\hat\Phi^\top \SigmaXh \hat\Phi} }^{-1}, \\
    \hat \Phi^\top \hatSigmaXh \hat\Phi &\succeq 0.9 \hat\Phi^\top \SigmaXh \hat\Phi \\
    \implies 2 \hat \Phi^\top \hatSigmaXh \hat\Phi &\succeq 0.9 \hat\Phi^\top \SigmaXh \hat\Phi + \hat \Phi^\top \hatSigmaXh \hat\Phi \\
    \implies (\hat \Phi^\top \hatSigmaXh \hat\Phi )^{-1} &\preceq 2 \paren{0.9 \hat\Phi^\top \SigmaXh \hat\Phi + \hat \Phi^\top \hatSigmaXh \hat\Phi}^{-1},
\end{align*}
such that
\begin{align*}
    &\norm{\paren{\sum_{i=1}^{T_1} \hat\Phi^\top \Xhi^\top \Xhi \hat\Phi }^{-1} \paren{\sum_{i=1}^{T_1} \hat\Phi^\top \Xhi^\top \whi }}_2^2 \\
    \lesssim\;& \frac{1}{ T_1 \lambda_{\min}\paren{\hat\Phi^\top \SigmaXh \hat\Phi}} \paren{\sigmahw^2 \log\paren{\frac{\det(\Sigma + \sum_{i=1}^{T_1} \hat\Phi^\top \Xhi^\top \Xhi \hat\Phi)}{\det(\Sigma)}} + 2\sigmahw^2 \log(1/\delta)} \\
    =\;& \frac{1}{ T_1 \lambda_{\min}\paren{\hat\Phi^\top \SigmaXh \hat\Phi}} \paren{\sigmahw^2 \log\paren{\det\paren{I_{\dtheta} + \paren{\sum_{i=1}^{T_1} \hat\Phi^\top \Xhi^\top \Xhi \hat\Phi}(\Sigma)^{-1})} + 2\sigmahw^2 \log(1/\delta)}} \\
    \leq\;&\frac{1}{ T_1 \lambda_{\min}\paren{\hat\Phi^\top \SigmaXh \hat\Phi}} \paren{\sigmahw^2 \dtheta \log\paren{1 + \frac{1.1}{0.9}} + 2\sigmahw^2 \log(1/\delta)} \\
    \lesssim\;& \sigmahw^2 \frac{\dtheta + \log(1/\delta)}{T_1 \lambda_{\min}(\hat\Phi^\top \SigmaXh \hat\Phi)}.
\end{align*}
\hfill$\blacksquare$

Having pulled out the misspecification error from term (a) and bounded the noise term (b), we may bound the contraction factor by instantiating \citet[Lemma A.12]{zhang2023meta}.
\begin{proposition}\label{prop: contraction factor bound}
    Assume that $\xhi$ are $B^2$-subgaussian for all $i \in [T_1]$, $h \in [H]$, and $\whi$ are $\sigmahw^2$-subgaussian for all $i \in [T_1]$. Define:
    \begin{align*}
        \bfThetamin \triangleq \lambda_{\min}\paren{\frac{1}{H} \sum_{h=1}^H \thetahstar \thetahstar^\top}, \;\;
        \bfThetamax \triangleq \lambda_{\max}\paren{\frac{1}{H} \sum_{h=1}^H \thetahstar \thetahstar^\top}
    \end{align*}
    If the following burn-in conditions hold:
    \begin{align*}
        d(\hat\Phi, \Phi_\star) &\leq \frac{1}{100} \sqrt{\frac{\bfThetamin}{\bfThetamax}} \; \max_h \frac{\lambda_{\min}\paren{\hat\Phi^\top \SigmaXh \hat\Phi}}{\norm{\hat\Phi^\top \SigmaXh \hatPhiperp}} \\
        T_1 &\gtrsim \max\curly{B^4 (\min\{\dtheta, \dx\} + \log(H/\delta)), \;\;\bfThetamin^{-1} \fracsumH \frac{\sigmahw^2}{\lambda_{\min}(\hat\Phi^\top \SigmaXh \hat\Phi)} \frac{\dtheta + \log(H/\delta)}{T_1}},
    \end{align*}
    where then for step-size satisfying $\eta \leq 0.956 \bfThetamax^{-1}$, with probability at least $1-\delta$, we have
    \begin{align*}
        \norm{I - \frac{\eta}{H}\sum_{h=1}^H \hatthetah \hatthetah^\top} &\leq \paren{1 - 0.954 \eta \bfThetamin}.
    \end{align*}
\end{proposition}
Therefore, we have established a bound on the contraction factor. However, we recall that the effective contraction factor is affected by the orthonormalization factor, which requires us to bound the noise-level of the \texttt{DFW}-gradient update.


\,\\\\
\textbf{Bounding the Noise Term and Orthonormalization Factor}

We now consider bounding the noise term:
\begin{align*}
     \frac{1}{H}\sum_{h=1}^H (\Xh \Xh^\top)^{-1} \Xh \Wh^\top \hatthetah^\top.
\end{align*}
Making use of the properties of the Kronecker product \eqref{eq: kronecker prod properties}, we have for each $h \in [H]$,
\begin{align*}
    (\Xh \Xh^\top)^{-1} \Xh \Wh^\top &= \paren{\sum_{i=1}^{T_2}\xhi\xhi^\top \otimes I_{\dy}}^{-1} \paren{\sum_{i=1}^{T_2} \paren{\xhi \otimes I_{\dy}}\whi} \\
    &= \paren{\paren{\sum_{i=1}^{T_2}\xhi\xhi^\top}^{-1} \otimes I_{\dy}} \VEC\paren{\sum_{i=1}^{T_2} \whi \xhi^\top} \\
    &= \VEC\paren{\paren{\sum_{i=1}^{T_2} \whi \xhi^\top}\paren{\sum_{i=1}^{T_2}\xhi\xhi^\top}^{-1}}.
\end{align*}
In particular, $\paren{\sum_{i=1}^{T_2} \whi \xhi^\top}\paren{\sum_{i=1}^{T_2}\xhi\xhi^\top}^{-1}$ is a least-squares-like noise term that we may bound with standard tools. Secondly, we need to bound the norm of $\hatthetah$, which follows straightforwardly from our earlier derivations.
\begin{lemma}\label{lem: ls head norm bound}
    If the following burn-in conditions are satisfied
    \begin{align*}
        d(\hat\Phi, \Phi_\star) &\leq \frac{2}{5}C \frac{\lambda_{\min}\paren{\hat\Phi^\top \SigmaXh \hat\Phi}}{\norm{\hat\Phi^\top \SigmaXh \hatPhiperp}} \\
        T_1 &\gtrsim \max\curly{B^4 (\min\{\dtheta, \dx\} + \log(H/\delta)),\;\; \max_h \frac{\sigmahw^2}{C^2\snorm{\thetahstar}^2\lambda_{\min}(\SigmaXh)} \paren{\dtheta + \log(H/\delta)}},
    \end{align*}
    for fixed $C > 0$. Then with probability at least $1 - \delta$
    \begin{align}
        \snorm{\hatthetah} &\leq (1+C) \snorm{\thetahstar} \text{ for all }h \in [H].
    \end{align}
\end{lemma}

\textit{Proof of \Cref{lem: ls head norm bound}:} we recall that $\hatthetah$ can be written as
\begin{align*}
    \hatthetah &=\hat \Phi^\top \Phi_\star \theta_\star + \paren{ \hat\Phi^\top \hatSigmaXh \hat\Phi }^{-1} \paren{\hat\Phi^\top \hatSigmaXh \hatPhiperp\hatPhiperp^\top\Phi_\star \theta_\star } + \paren{\sum_{i=1}^{T_1} \hat\Phi^\top \hatSigmaXh \hat\Phi }^{-1} \paren{\sum_{i=1}^{T_1} \hat\Phi^\top \Xhi^\top \whi }.
\end{align*}
Applying the triangle equality, and applying \Cref{lem: rep-conditioned cov conc bound}, if $T_1 \gtrsim B^4 (\min\{\dtheta, \dx\} + \log(1/\delta))$, with probability at least $1 - \delta$, we have
\begin{align*}
    \snorm{\hatthetah} &\leq \snorm{\thetahstar} + \frac{\norm{\hat\Phi^\top \hatSigmaXh \hatPhiperp}}{\lambda_{\min}\paren{\hat\Phi^\top \hatSigmaXh \hat\Phi}} d(\hat\Phi, \Phi_\star) \snorm{\thetahstar} + \norm{\paren{\sum_{i=1}^{T_1} \hat\Phi^\top \hatSigmaXh \hat\Phi }^{-1} \paren{\sum_{i=1}^{T_1} \hat\Phi^\top \Xhi^\top \whi }} \\
    &\leq \snorm{\thetahstar} + \frac{1.1 \norm{\hat\Phi^\top \SigmaXh \hatPhiperp}}{0.9 \lambda_{\min}\paren{\hat\Phi^\top \SigmaXh \hat\Phi}} d(\hat\Phi, \Phi_\star) \snorm{\thetahstar} + \sqrt{\sigmahw^2 \frac{\dtheta + \log(1/\delta)}{T_1 \lambda_{\min}(\hat\Phi^\top \SigmaXh \hat\Phi)}}.
\end{align*}
Inverting the second and third terms for the burn-in conditions for $d(\hat\Phi, \Phi_\star)$ and $T_1$ yields the desired bound $\snorm{\hatthetah} \leq (1+C) \snorm{\thetahstar}$. Union bounding over $h \in [H]$ yields the final result.

\hfill$\blacksquare$

Combining this with an application of a matrix Hoeffding's inequality (see e.g.\ \citet[Lemma A.5]{zhang2023meta}), we get the following bound on the noise term:
\begin{proposition}[\texttt{DFW} noise term bound]\label{prop: dfw noise term bound}
    Assume that $\xhi$ are $B^2$-subgaussian for all $i \in [T_1+T_2]$, $h \in [H]$, and $\whi$ are $\sigmahw^2$-subgaussian for all $i \in [T_1 + T_2]$. If the following burn-in conditions are satisfied
    \begin{align*}
        d(\hat\Phi, \Phi_\star) &\leq \frac{2}{5}C \frac{\lambda_{\min}\paren{\hat\Phi^\top \SigmaXh \hat\Phi}}{\norm{\hat\Phi^\top \SigmaXh \hatPhiperp}} \\
        T_1 &\gtrsim \max\curly{B^4 (\min\{\dtheta, \dx\} + \log(H/\delta)), \frac{\sigmahw^2}{C^2\snorm{\thetahstar}^2\lambda_{\min}\paren{\hat\Phi^\top \SigmaXh \hat\Phi}} \paren{\dtheta + \log(H/\delta)}} \\
        T_2 &\gtrsim B^4 (\dx + \log(H/\delta)),
    \end{align*}
    then with probability at least $1 - \delta$, the following bound holds:
    \begin{align*}
        \norm{\frac{1}{H}\sum_{h=1}^H (\Xh \Xh^\top)^{-1} \Xh \Wh^\top \hatthetah^\top} &\leq (1+C) \sigmaavg \sqrt{\frac{\dy\dx + \log(H/\delta)}{H T_2} \log\paren{\frac{d_x}{\delta}}},
    \end{align*}
    where $\sigmaavg \triangleq \sqrt{\fracsumH \frac{\sigmahw^2 \snorm{\thetahstar}^2}{\lambda_{\min}\paren{\hat\Phi^\top \SigmaXh \hat\Phi}}}$ is the \emph{task-averaged} noise-level.
\end{proposition}

\textit{Proof of \Cref{prop: dfw noise term bound}:} we follow the proof structure in \citet[Proposition A.2]{zhang2023meta}. By observing that $\paren{\sum_{i=1}^{T_2} \whi \xhi^\top}\paren{\sum_{i=1}^{T_2}\xhi\xhi^\top}^{-1/2}$ is a $\dy \times \dx$ matrix-valued self-normalized martingale (see \citet[Theorem 4.1]{ziemann2023tutorial}), when $T_2 \gtrsim B^4 (\dx + \log(1/\delta))$, we have with probability at least $1 - \delta$ for a fixed $h \in [H]$
\begin{align*}
    \norm{(\Xh \Xh^\top)^{-1/2} \Xh \Wh^\top}^2_2 &= \norm{\VEC\paren{\paren{\sum_{i=1}^{T_2} \whi \xhi^\top}\paren{\sum_{i=1}^{T_2}\xhi\xhi^\top}^{-1/2}}}^2_2 \\
    &= \norm{\paren{\sum_{i=1}^{T_2} \whi \xhi^\top}\paren{\sum_{i=1}^{T_2}\xhi\xhi^\top}^{-1/2}}_F^2 \\
    &\leq \sigmahw^2 \frac{\dy\dx + \log(1/\delta)}{T_2 \lambda_{\min}(\SigmaXh)}.
\end{align*}
Therefore, combining with \Cref{lem: ls head norm bound} and union bounding over $h \in [H]$, we have with probability at least $1 - \delta$
\begin{align*}
    \norm{(\Xh \Xh^\top)^{-1} \Xh \Wh^\top \hatthetah^\top}_2 &\lesssim (1+C) \frac{\sigmahw \snorm{\thetahstar}}{\sqrt{\lambda_{\min}(\SigmaXh)}} \sqrt{\frac{\dy\dx + \log(H/\delta)}{T_2}}, \quad \forall h \in [H].
\end{align*}
Conditioning on this boundedness event and using the fact that $(\Xh \Xh^\top)^{-1/2} \Xh \Wh^\top$ are zero-mean across $h \in [H]$, we may instantiate a matrix Hoeffding inequality (see e.g.\ \citet[Lemma A.5]{zhang2023meta}):
\begin{align*}
    \norm{\frac{1}{H}\sum_{h=1}^H (\Xh \Xh^\top)^{-1} \Xh \Wh^\top \hatthetah^\top} &\leq (1+C) \sigmaavg \sqrt{\frac{\dy\dx + \log(H/\delta)}{H T_2} \log\paren{\frac{\dx}{\delta}}}.
\end{align*}
\hfill $\blacksquare$

We recall that \texttt{DFW} involves pulling out an orthonormalization factor $R$ \eqref{eq: DFW contraction decomp}. Having bounded the contraction factor in \Cref{prop: contraction factor bound} and noise term in \Cref{prop: dfw noise term bound}, we may now follow the recipe in \citet[Lemma A.13]{zhang2023meta} to yield the following bound on the orthonormalization factor.
\begin{proposition}[Orthonormalization factor bound]\label{prop: DFW ortho factor bound}
    Let the following burn-in conditions hold:
    \begin{align*}
        d(\hat\Phi, \Phi_\star) &\leq\frac{1}{100} \sqrt{\frac{\bfThetamin}{\bfThetamax}} \; \max_h \frac{\lambda_{\min}\paren{\hat\Phi^\top \SigmaXh \hat\Phi}}{\norm{\hat\Phi^\top \SigmaXh \hatPhiperp}} \\
        T_1 &\gtrsim \max\curly{B^4 (\min\{\dtheta, \dx\} + \log(H/\delta)), \;\;\barsigmatheta^2 (\dtheta + \log(H/\delta))}, \\
        T_2 &\gtrsim \max\curly{B^4 \paren{\dx + \log(H/\delta)},  \bfThetamin^{-1} \blue{\frac{\sigmaavg^2}{H}} (\dy\dx + \log(H/\delta)) \log\paren{\frac{\dx}{\delta}}  },
    \end{align*}
    where $\barsigmatheta^2 \triangleq \max\curly{\max_h \frac{\sigmahw^2}{\snorm{\thetahstar}^2\lambda_{\min}\paren{\hat\Phi^\top \SigmaXh \hat\Phi}}, \;\; \bfThetamin^{-1}\fracsumH \frac{\sigmahw^2}{\lambda_{\min}(\hat\Phi^\top \SigmaXh \hat\Phi) }}$. Then, given $\eta \leq 0.956 \bfThetamax^{-1}$, with probability at least $1 - \delta$, we have the following bound on the orthogonalization factor $R$:
    \begin{align*}
        \norm{R^{-1}} &\leq \paren{1 - 0.0575 \;\eta \bfThetamin}^{-1/2},
    \end{align*}
\end{proposition}

We now combine \Cref{prop: contraction factor bound}, \Cref{prop: dfw noise term bound}, and \Cref{prop: DFW ortho factor bound} to yield the representation error improvement from running one iteration of \texttt{DFW}.
\begin{theorem}
    Assume that $\xhi$ are $B^2$-subgaussian for all $i \in [T_1]$, $h \in [H]$, and $\whi$ are $\sigmahw^2$-subgaussian for all $i \in [T_1]$. Let the following burn-in conditions hold:
    \begin{align*}
        d(\hat\Phi, \Phi_\star) &\leq\frac{1}{100} \sqrt{\frac{\bfThetamin}{\bfThetamax}} \; \max_h \frac{\lambda_{\min}\paren{\hat\Phi^\top \SigmaXh \hat\Phi}}{\norm{\hat\Phi^\top \SigmaXh \hatPhiperp}} \\
        T_1 &\gtrsim \max\curly{B^4 (\min\{\dtheta, \dx\} + \log(H/\delta)), \;\;\barsigmatheta^2 (\dtheta + \log(H/\delta))}, \\
        T_2 &\gtrsim \max\curly{B^4 \paren{\dx + \log(H/\delta)},  \bfThetamin^{-1} \blue{\frac{\sigmaavg^2}{H}} (\dy\dx + \log(H/\delta)) \log\paren{\frac{\dx}{\delta}}  }.
    \end{align*}
    Then, given step-size satisfying $\eta \leq 0.956 \bfThetamax^{-1}$, with probability at least $1 - \delta$, running one iteration of \texttt{DFW} (\Cref{alg: dfw}) yields an updated representation $\hat\Phi_+$ satisfying:
    \begin{align*}
        d(\hat\Phi_+, \Phi_\star) &\leq \paren{1 - 0.897 \eta \bfThetamin} d(\hat\Phi, \Phi_\star) + C \cdot \sigmaavg \sqrt{\frac{\dy\dx + \log(H/\delta)}{H T_2} \log\paren{\frac{\dx}{\delta}}},
    \end{align*}
    for a universal numerical constant $C > 0$.
\end{theorem}

\subsection*{Instantiating \texttt{DFW} to Linear System Identification}

Having established representation error guarantees for \texttt{DFW} on (independent) sub-Gaussian data, we now instantiate to our linear system identification setting. We alias the following variables:
\begin{align*}
    \yhi &\leftarrow x^{(h)}_{t+1}, \;\;(\dy \leftarrow \dx) \\
    \xhi &\leftarrow \bmat{x^{(h)}_t \\ u^{(h)}_t},\;\;(\dx \leftarrow \dx + \du) \\
    B^2 &\leftarrow x_b^2 \log(T) \\
    \sigmahw &\leftarrow \sigma^2 \;\;\forall h \in [H] \\
    \SigmaXh &\leftarrow \E \brac{ \frac{1}{t} \sum_{s=1}^t \bmat{x^{(h)}_s \\ u^{(h)}_s} \bmat{x^{(h)}_s \\ u^{(h)}_s}^\top \Big\vert\; x^{(h)}_1} =\Sigma^t_h(K^{(h)},\sigma_u,x^{(h)}_1), \;\;(\text{see \Cref{eq: sysid emp cov matrices}})
\end{align*}
and recall the following instantiated \texttt{DFW}-related definitions:
\begin{align*}
    \bfThetamin &\triangleq \lambda_{\min}\paren{\frac{1}{H} \sum_{h=1}^H \thetahstar \thetahstar^\top} \\
    \bfThetamax &\triangleq \lambda_{\max}\paren{\frac{1}{H} \sum_{h=1}^H \thetahstar \thetahstar^\top} \\
     \sigmaavg &\triangleq \sqrt{\fracsumH \frac{\sigma^2 \snorm{\thetahstar}^2}{\lambda_{\min}\paren{\hat\Phi^\top \paren{\Sigma^t_h(K^{(h)},\sigma_u,x^{(h)}_1)\otimes I_{\dx}} \hat\Phi}}} \\
     \barsigmatheta^2 &\triangleq \max\Bigg\{\max_h \frac{\sigma^2}{\snorm{\thetahstar}^2\lambda_{\min}\paren{\hat\Phi^\top\paren{\Sigma^t_h(K^{(h)},\sigma_u,x^{(h)}_1)\otimes I_{\dx}} \hat\Phi}}, \\
     &\qquad\qquad\bfThetamin^{-1}\fracsumH \frac{\sigma^2}{\lambda_{\min}(\hat\Phi^\top \paren{\Sigma^t_h(K^{(h)},\sigma_u,x^{(h)}_1)\otimes I_{\dx}} \hat\Phi) } \Bigg\}
\end{align*}
In order to go from independent covariates to handle dependence, we use a mixing-time argument (see e.g.\ \citet[Appendix B]{zhang2023meta} for further details).

\begin{definition}\label{def: mixing time}
    For given stabilizing controllers $K^{(h)}$, $h \in [H]$, i.e.\ $\rho(A_\star^{(h)} + B_\star^{(h)}K^{(h)}) < 1$, define $\Gamma^\vee_{K} > 0$ and $\mu^{\vee}_K \in (0,1)$ as constants such that for all $h \in [H]$, $\snorm{(A_\star^{(h)} + B_\star^{(h)}K^{(h)})^t} \leq \Gamma^\vee_{K} \paren{\mu^{\vee}_K}^t$ for any $t \geq 0$.\footnote{Such constants are guaranteed to exist by, e.g.\ Gelfand's Formula \citep{horn2012matrix}.}
\end{definition}
To express the mixing-time parameters $\Gamma^\vee_{K}, \mu^{\vee}_K$ explicitly in terms of control-theoretic quantities, we have the following lemma.
\begin{lemma}
    Suppose $A \in \R^{\dx\times\dx}$ satisfies $\rho(A) < 1$ and let $P = \dlyap(A, Q)$ for $Q \succ 0$. It holds that $\norm{A^t} \leq \norm{P} \paren{1 - \frac{1}{\norm{P}}}^t$.
\end{lemma}
As a result, let us define the mixing-time $\tau_{\mathsf{mix}}(\delta) \triangleq P_K^\vee \log \paren{T \norm{P_K} \sqrt{x_b^2 \log T + \dx (P_K^\vee)^2}/\delta}$.
\begin{assumption}[\texttt{DFW} burn-in]\label{assumption: dfw burn-in}
    Consider running \Cref{alg: dfw} on data generated from arbitrary initial states $x_1^{(1)}, \dots, x_t^{(H)}$, norm-bounded by $x_b \sqrt{\log(T)}$ (see Line~\ref{line: uniform bound}), by closed loop systems under stabilizing controllers $K^{(1)}, \dots, K^{(H)}$ with exploratory noise $\sigma_u g_t$, $g_t \iidsim \calN(0,I_{\du})$, with $\sigma_u \leq 1$, and representation $\hat\Phi$. Let the number of gradient steps $N$ satisfy $N \geq \log(2)/{\log\paren{\frac{1}{1 - 0.897\eta\bfThetamin}}}$, and subtrajectory lengths $t_1,t_2$ satisfy $N(t_1 + t_2) \leq t$.
    For a given failure probability $\delta \in (0,1)$, let the following hold on the epoch length and systems $h \in [H]$:
    \begin{align*}
        &\rank\paren{\sumH \thetahstar \thetahstar^\top} = \dtheta \\
        &d(\hat\Phi, \Phi_\star) \leq \frac{1}{100} \min_h \frac{\lambda_{\min}(\hat\Phi^\top \Sigma_h^t(K^{(h)},\sigma_u,x^{(h)}_1) \hat \Phi)}{\norm{\Sigma_h^t(K^{(h)},\sigma_u,x^{(h)}_1)}} \sqrt{\frac{\bfThetamin}{\bfThetamax}} \\
        &t_1 \gtrsim \tau_{\mathsf{mix}}(\delta)\; \max\curly{x_b^4\log^2(T) (\min\{\dtheta, \dx+\du\} + \log(H/\delta)), \;\;\barsigmatheta^2 (\dtheta + \log(H/\delta))} \\
        &t_2 \gtrsim \tau_{\mathsf{mix}}(\delta)\; \max\curly{x_b^4\log^2(T) \paren{\dx+\du + \log(H/\delta)},  \bfThetamin^{-1} \frac{\sigmaavg^2}{H}(\dx(\dx+\du) + \log(H/\delta)) \log\paren{\frac{\dx+\du}{\delta}}  }.
    \end{align*}
\end{assumption}
The burn-in conditions for \texttt{DFW} may be stated in terms of quantities which are polynomial in system parameters, up to the levels of exploration which are balanced in the downstream regret analysis. In particular, we may instantiate the lower bound $\lambda_{\min}\paren{\Sigma^t_h(K^{(h)},\sigma_u,x^{(h)}_1)} \succeq \frac{\sigma_u^2 }{2(1+2\norm{K}^2 + \sigma_u^2)}  I$ (see \Cref{lem: covariance facts}).

We now state a bound on the improvement of the subspace distance after running \Cref{alg: dfw}.
\begin{theorem}[\texttt{DFW} guarantee]
    \label{thm: dfw bound formal}
    Let Assumption~\ref{assumption: dfw burn-in} hold for given $\delta \in (0,1)$. Then, 
    with probability at least $1 - \delta$ running \Cref{alg: dfw} yields the following guarantee on the updated representation $\hat\Phi \to \hat\Phi_N$:
    \begin{align*}
        d(\hat \Phi_{N}, \Phi_\star) &\leq \paren{1 - 0.897\eta\bfThetamin}^N d(\hat \Phi, \Phi^\star) + \Ccontract\frac{\sqrt{N}}{\sigma_u \sqrt{t H}}
    \end{align*}
    where
    \begin{align*}
        \Ccontract &= \frac{\Kbaravg}{1 - \sqrt{2}\rho^N}\\
        \Kbaravg &= \sqrt{\avgsumH \sigma^2 \snorm{\thetahstar}^2 (2 + \snorm{K^{(h)}}^2) \dx(\dx + \du)} \cdot  \mathrm{polylog}(\dx, \du, H, 1/\delta).
    \end{align*}
\end{theorem}





\section{High probability Bounds on the Success events}
\label{s: success event probability bounds}

We begin by presenting several auxiliary lemmas from prior work.

\subsection{Auxillary Lemmas}
\begin{lemma}(Noise bound (Lemma 13 of \citep{lee2023nonasymptotic}))
    \label{lem: noise bound}
    Let $\delta \in (0,1)$. For any task $h \in [H]$, it holds that 
    \begin{align*}
        \max_{1\leq t \leq T} \norm{\bmat{w^{(h)}_t \\ g^{(h)}_t}} \leq 4\sigma \sqrt{ (\dx+\du) \log\frac{T}{\delta}},
    \end{align*}
    with probability at least $1-\delta$.
\end{lemma}

For any task $h \in [H]$, we define the empirical covariance matrix conditioned on the initial state $x^{(h)}_1$ as follows:
\begin{align}\label{eq: sysid emp cov matrices}
    \begin{split}
        \Sigma^t_h(K^{(h)},\sigma_u,x^{(h)}_1) &\triangleq \E \brac{ \frac{1}{t} \sum_{s=1}^t \bmat{x^{(h)}_s \\ u^{(h)}_s} \bmat{x^{(h)}_s \\ u^{(h)}_s}^\top \vert x^{(h)}_1} \mbox{ and }  \\
        \bar \Sigma^t_h(K^{(h)},\sigma_u,x^{(h)}_1) &\triangleq \E \brac{ \frac{1}{t} \sum_{s=1}^{t} \paren{\bmat{x^{(h)}_s \\ u^{(h)}_s} - \E \brac{\bmat{x^{(h)}_s \\ u^{(h)}_s} \vert x^{(h)}_1} }\paren{\bmat{x^{(h)}_s \\ u^{(h)}_s} - \E \brac{\bmat{x^{(h)}_s \\ u^{(h)}_s }\vert x^{(h)}_{1} }}^\top}.
    \end{split}
\end{align}
where $\bar \Sigma^t_h(K^{(h)},\sigma_u,x^{(h)}_1)$ denotes the centered empirical covariance matrix from rolling out system $h$ under control inputs $u_s^{(h)} = K^{(h)} x_s^{(h)} + \sigma_u g_s^{(h)}$ for $t$ steps starting from an arbitrary initial state $x_1^{(h)}$.

\begin{lemma}(Epoch-wise covariance bounds (Lemma 2 of \citep{lee2023nonasymptotic}))
    \label{lem: covariance facts}
    For $t \geq 2$ and task $h \in [H]$, where we denote $K^{(h)} = K$, $A^{(h)} = A$, $B^{(h)} = B$, $\Sigma^t_h(K^{(h)},\sigma_u,x^{(h)}_1) = \Sigma^t(K^{(h)},\sigma_u,x^{(h)}_1)$, and $\bar \Sigma^t_h(K^{(h)},\sigma_u,x^{(h)}_1) = \bar \Sigma^t(K^{(h)},\sigma_u,x^{(h)}_1)$ we have
    \begin{enumerate}
        \item $\bar \Sigma^t(K,\sigma_u,x_1) = \frac{1}{t}\sum_{s=0}^{t-2} \sum_{j=0}^s \bmat{I \\ K} A_{K}^j (\sigma_u^2 B^\star \paren{B^\star}^\top +  I)   \paren{A_K^j}^\top \bmat{I \\  K}^\top + \bmat{0 & 0 \\ 0 & \sigma_u^2 I_{\du}}$
        \item $ \Sigma^t(K,\sigma_u,x_1) = \bar \Sigma_k + \frac{1}{t} \sum_{s=0}^{t-1} \bmat{I \\ K} A_{K}^{s} x_1 x_1^\top \paren{A_K^{s}}^\top \bmat{I \\  K}^\top$
        \item $\Sigma^t(K,\sigma_u,x_1)  \succeq \bar \Sigma^t(K,\sigma_u,x_1)  \succeq \frac{\sigma_u^2 }{2(1+2\norm{K}^2 + \sigma_u^2)}  I$ 
    \end{enumerate}
    \begin{enumerate}\addtocounter{enumi}{3}
        \item $ \Sigma^t(K,\sigma_u,x_1)  \preceq (1+  \norm{P_K}\frac{\norm{ x_1^2}}{ t-1 }) \bar \Sigma^t(K,\sigma_u,x_1)$  
        \item $\norm{\bar \Sigma^t(K,\sigma_u,x_1)} \leq 5 \norm{P_K}^2  \Psi_{B^\star }^2$. 
    \end{enumerate}
\end{lemma}

\begin{lemma}(State bounds (Lemma 15 of \citep{lee2023nonasymptotic}))
    \label{lem: state rollout bounds}
    Consider rolling out the system $x_{s+1} = A_{\star} x_s + B_{\star} u_s +w_s$ from initial state $x^{(h)}_1$ for $t$ time-steps under the control action $u_s = K x_s + \sigma_u g_s$ where $K$ is stabilizing and $\sigma_u \leq 1$. Suppose 
    \begin{itemize}
        \item $\norm{x_1} \leq 16 \norm{P_{K_0}}^{3/2} \Psi_{B^{\star}} \max_{1\leq t \leq T} \norm{\bmat{w_t \\ g_t}}$
        \item $\norm{P_K} \leq 2 \norm{P_{K_0}}$
        \item $t \geq \log_{\paren{1 - \frac{1}{\norm{P_K}}}}\paren{\frac{1}{4\norm{P_K}}}+1$.
    \end{itemize}  Then for $s=1, \dots, t$
    \begin{align*}
        \norm{x_s} \leq 40 \norm{P_{K_0}}^2  \Psi_{B^{\star}} \max_{1 \leq t \leq T} \norm{\bmat{w_t \\ g_t}}.
    \end{align*}
    Furthermore, 
    \begin{align*}
        \norm{x_t} \leq 16 \norm{P_{K_0}}^{3/2} \Psi_{B^{\star}}\max_{1\leq t \leq T} \norm{\bmat{w_t \\ g_t}}.
    \end{align*}
\end{lemma}

\begin{theorem}(Theorem 3 of \cite{simchowitz2020naive})
    \label{lem: CE closeness}
    Define $\varepsilon^{(h)} \triangleq \frac{1}{2916 \norm{P_\star^{(h)}}}$ for any task $h \in [H]$. As long as 
    $$
        \norm{\bmat{\hat A^{(h)} & \hat B^{(h)}}-\bmat{A^{(h)}_\star & B^{(h)}_\star}}_F^2 \leq \varepsilon, 
    $$
    we have that $P^{(h)}_{\hat K} \preceq \frac{21}{20} P^{(h)}_\star$, $\norm{\hat K^{(h)} - K^{(h)}_\star}\leq \frac{1}{6 \norm{P^{(h)}_\star}^{3/2}}$, and 
    \begin{align*}
        \calJ^{(h)}(\hat K^{(h)})  - \calJ^{(h)}(K^{(h)}_\star) \leq 142 \norm{P^{(h)}_\star}^8 \norm{\bmat{\hat A^{(h)} & \hat B^{(h)}}-\bmat{A^{(h)}_\star & B^{(h)}_\star}}_F^2.
    \end{align*}
\end{theorem}

Using the above lemmas and theorems, we can mirror the arguments from Appendix C of \cite{lee2023nonasymptotic} to show that the events of success $\calE_{\mathsf{success},1}$ and $\calE_{\mathsf{success},2}$ hold under high probability.

\subsection{High Probability Bound on Success Event 1 (Hard to identify parameters)}
\begin{lemma}
    \label{lem: success event bound 1}
    Running \Cref{alg: ce with exploration} with the arguments defined in \Cref{thm: regret bound naive exploration}, the event $\calE_{\mathsf{success},1} $ holds with probability at least $1-T^{-2}$.
    \begin{proof}

To show that the success event $\mathcal{E}_{\text{success},1}$ holds under probability $1 - T^{-2}$ we can use an induction approach. For this purpose, we show, with high probability, that for every epoch $k \in \left[k_{\text {fin }}\right]$, Algorithm \ref{alg: ce with exploration} does not abort, i.e., the state and controller bounds are satisfied, the least-square estimation error is maintained small and scales according to the bound in $\mathcal{E}_{\text {est},1}$, and the learned common representation contracts towards its optimal as in $\mathcal{E}_{\text {cont}}$. We begin our analysis by studying the first epoch. 

\vspace{0.2cm}

\noindent \textbf{Base case:} We consider the first epoch $k=1$ as the base case of the induction approach. For convenience we assume that $x^{(h)}_1 = 0$, for all tasks $h \in [H]$. However, it is worth noting that the proof below can be readily extended to bounded non-zero initial states.

\begin{itemize}
    \item \textbf{The bounds on $\|x^{(h)}_t\|^2$ for $t =\{0,1,\ldots,\tau_1\}$ and $K^{(h)}_0$ are not violated:} We first show that, with high probability, the state and controller bounds are not violated during the first epoch. To do so we have to bound the worst-case behavior of the process and exploratory noises, which can be accomplished by using Lemma \ref{lem: noise bound} to obtain
      \begin{align}\label{eq:noise_bound_mt}
         \max_{1\leq t \leq T} \norm{\bmat{w^{(h)}_t \\ g^{(h)}_t}} \leq 4\sigma\sqrt{3 (\dx + \du) \log (3HT)}.  
    \end{align}
with probability $1 - \frac{1}{3}T^{-2}$, for all tasks $h \in [H]$. Then, since the initial state norm (i.e., $\|x^{(h)}_1\| = 0$) satisfy 
$$\norm{x^{(h)}_1} \leq 16 (\pzeromax)^{3/2} \psibmax \max_{1\leq t \leq T} \norm{\bmat{w^{(h)}_t \\ g^{(h)}_t}},$$ 
and the initial epoch length can selected according to $\tau_1 \geq \frac{ c \log \frac{1}{\pstarmin}}{\log\paren{1- \frac{1}{ \pstarmin}}}$, for a sufficiently large constant $c$. We then may use Lemma \ref{lem: state rollout bounds} to write
\begin{align}\label{eq:state_norm_first_epoch}
        \norm{x^{(h)}_t} \leq 40 (\pzeromax)^2\psibmax \max_{1\leq t \leq T} \norm{\bmat{w^{(h)}_t \\ g^{(h)}_t}}, \;\ \forall t = \{0,1,\ldots,\tau_1\},
\end{align}
and by using \eqref{eq:noise_bound_mt} in \eqref{eq:state_norm_first_epoch} we have
\begin{align*}
        \norm{x^{(h)}_t}^2 \leq 76800 (\pzeromax)^4 (\psibmax)^2 \sigma^2 (\dx + \du)\log (3HT), \;\ \forall t = \{0,1,\ldots,\tau_1\}
\end{align*}
with probability $1- \frac{1}{3}T^{-3}$, $\forall h \in [H]$, which implies that  $\norm{x^{(h)}_t}^2 \leq x^2_b\log T$. For the controller bound, we can notice that $\|K^{(h)}_0\|^2 \leq \pzeromax \leq 2\pzeromax$, which leads to $\|K^{(h)}_0\| \leq K_b$.  Therefore, we define the event where the state and controller bounds are satisfied for the first epoch and obtain that $\mathcal{E}_{\text{bound},1}$ holds under high probability $1-\frac{1}{3}T^{-2}$.
\vspace{0.2cm}

\item \textbf{Controlling the least-square estimation error:} To control the estimation error at the first epoch, one may exploit Theorem \ref{thm: ls estimation error}. Note that a condition $\tau_{\mathsf{warm\_up}} \geq \sigma^4 \pzeromax (\psibmax)^2  (\dx+\du)$ implies that $\tau_1 \geq c \tau_{\text{ls}}(K_0, 0, \frac{1}{3}T^{-3})$, for a sufficiently large constant $c$, which satisfy the condition of Theorem \ref{thm: ls estimation error} to obtain
\begin{equation}
\begin{aligned}\label{eq:estimation_error_first_epoch}
        &\norm{\bmat{\hat A^{(h)}_1 & \hat B^{(h)}_1} - \bmat{A^{(h)}_\star & B^{(h)}_\star}}_F^2 \lesssim \frac{\dtheta \sigma^2 \log(HT)  }{\tau_1   \underset{h=1,\dots, H}{\min}\lambda_{\min}(\hat \Phi^\top_1\paren{\bar \Sigma^{\tau_1}_h(K^{(h)}_0,\sigma_1,0) \otimes I_{\dx}} \hat \Phi_1)}  \\
        &+ \paren{1+\frac{\sigma^4 (\pzeromax)^7 (\psibmax)^6 \paren{\dx+\du + \log (HT)}}{\tau_1 \underset{h=1,\dots, H}{\min}\lambda_{\min}(\hat \Phi^\top_1 \paren{\bar \Sigma^{\tau_1}_h(K^{(h)}_0,\sigma_1,0) \otimes I_{\dx}} \hat \Phi_1)^2 } } \frac{(\pzeromax)^2 (\psibmax)^2 d(\hat \Phi_1, \Phi_\star)^2 (\thetamax)^2}{ \underset{h=1,\dots, H}{\min}\lambda_{\min}(\hat \Phi^\top_1 \paren{\bar \Sigma^{\tau_1}_h(K^{(h)}_0,\sigma_1,0) \otimes I_{\dx}} \hat \Phi_1)} .
    \end{aligned}
    \end{equation}
with probability $1 - \frac{1}{3}T^{-3}$, for all tasks $h \in [H]$. We note that the rate of the decay in the estimation error is controlled by the minimum eigenvalue of the input-state covariance matrix. Then, we may use the third point of Lemma \ref{lem: covariance facts} to obtain 
\newcommand{\kzeromax}{\mathopen{}K_0^\vee\mathclose{}}
\begin{align}
        \label{eq: covariance lower bound continual exploration mt}
        \underset{h=1,\dots, H}{\min} \lambda_{\min}(\hat \Phi^\top_1 \paren{\bar \Sigma^{\tau_1}_h(K^{(h)}_0,\sigma_1,0) \otimes I_{\dx}} \hat \Phi_1) \geq \frac{\sigma_1^2}{2(2 + 2  (\kzeromax)^2)} \geq \frac{\sigma_1^2}{8 \pzeromax},
\end{align}
where $\kzeromax = \underset{h=1,\dots, H}{\max} \norm{K^{(h)}_0}$ and the final inequality follows from the fact that $2 + 2(\kzeromax)^2 \leq 2 + 2\pzeromax\leq 4 \pzeromax$ and $\pzeromax \geq 1$. We then use \eqref{eq: covariance lower bound continual exploration mt} in  \eqref{eq:estimation_error_first_epoch} to obtain
\begin{align*}
        &\norm{\bmat{\hat A^{(h)}_1 & \hat B^{(h)}_1} - \bmat{A^{(h)}_\star & B^{(h)}_\star}}_F^2 \lesssim \frac{\dtheta \sigma^2 (\pzeromax) }{\tau_1 \sigma_1^2} \log\paren{HT} \\
        &\quad+ \paren{1+ \frac{\sigma^4 (\pzeromax)^9 (\psibmax) \paren{\dx+\du + \log\paren{HT}}}{\tau_1 \sigma_1^4 } } \frac{(\pzeromax)^3 (\psibmax)^2 d(\hat \Phi_1, \Phi_\star)^2 (\thetamax)^2}{\sigma_1^2}.
    \end{align*}
and from $\sigma_1^2 \geq \tau_1^{-1/4} H^{-1/5}$ along with the fact that $\sqrt{\tau_1} \geq \log(HT)$ (by the choice of $\tau_{\mathsf{warm\,up}}$), we have that
\begin{align*}
        &\norm{\bmat{\hat A^{(h)}_1 & \hat B^{(h)}_1} - \bmat{A^{(h)}_\star & B^{(h)}_\star}}_F^2 \lesssim \frac{\dtheta \sigma^2 (\pzeromax) }{\tau_1 \sigma_1^2} \log\paren{HT} \\
        &\quad+ H^{2/5} \sigma^4 \frac{\du}{\dtheta} (\pzeromax)^{12} (\psibmax)^8 \paren{\dx+\du}(\thetamax)^2  \frac{d^2(\hat \Phi_1, \Phi_\star)}{\sigma^2_1}.
    \end{align*}
Then, by defining $\beta_1 \triangleq C_{\mathsf{bias},1}  \sigma^4 (\pzeromax)^{12} (\psibmax)^8 (\thetamax)^2 (\dx+\du)\frac{\du}{\dtheta}$ we obtain 

\begin{align*}
       &\norm{\bmat{\hat A^{(h)}_1 & \hat B^{(h)}_1} - \bmat{A^{(h)}_\star & B^{(h)}_\star}}_F^2 \leq C_{\text{est,1}}\frac{\dtheta \sigma^2 (\pzeromax) }{\tau_1 \sigma_1^2} \log\paren{HT} + \frac{\beta_1 H^{2/5} d^2(\hat \Phi_1, \Phi_\star)}{\sigma^2_1}.
    \end{align*}

Therefore, by defining the event $\mathcal{E}_{\text{ls},1}$ where the above least-square estimation error at the first epoch holds, we have that  $\mathcal{E}_{\text{ls},1}$ holds under probability $1 - \frac{1}{3}T^{-3}$, for all tasks $h \in [H]$. 

\vspace{0.2cm}

\item \textbf{Controlling the error in the learned representation:} For the first epoch, we initialize the representation as $\hat{\Phi}_0$. Then, Algorithm \ref{alg: ce with exploration} play $K^{(h)}_0$ for all tasks $h \in [H]$ to collect a multi-task dataset that is leveraged to compute $\hat\Phi_1$ via Algorithm \ref{alg: dfw}. Therefore, we can set $ \tau_1 \geq c\tau_{\mathsf{dfw}}$, for a sufficiently large constant $c$ and leverage Assumption~\ref{asmp: upper bound on representation error exp} to apply \Cref{thm: dfw bound formal} with $\delta = \frac{1}{3} T^{-3}$ and obtain
\begin{align*}
        d\paren{\hat\Phi_+, \Phi_\star} &\leq \rho^N d\paren{\hat\Phi, \Phi_\star} + \frac{\Ccontract \sqrt{N}\log(HT)}{\sqrt{H\tau_1\sigma_1^2}}.
\end{align*}
Denote this event by $\mathcal{E}_{\text{c},1}$. 
\end{itemize}
\vspace{0.2cm}
\noindent \textbf{Induction step:} We now introduce an induction step to extend our analysis for every epoch. For this purpose, based on the first epoch one may establish the following inductive hypothesis:

\begin{align}\label{eq:inductive_hypotesis_state_control_bounds}
   \text{\textbf{Bounded state}:}\;\ \norm{x^{(h)}_{\tau_k}} \leq 16 (\pzeromax)^{3/2} (\psibmax) \max_{1\leq t \leq T} \norm{\bmat{w^{(h)}_t \\ g^{(h)}_t}},
\end{align}

\begin{align}\label{eq:inductive_hypotesis_estimation_error}
       \text{\textbf{Least-square error}:}\;\ \norm{\bmat{\hat A^{(h)}_k & \hat B^{(h)}_k} - \bmat{A^{(h)}_\star & B^{(h)}_\star}}_F^2 \leq C_{\text{est,1}}\frac{\dtheta \sigma^2 (\pzeromax) }{\tau_k \sigma_k^2} \log\paren{HT} + \frac{\beta_1 H^{2/5} d^2(\hat \Phi_k, \Phi_\star)}{\sigma^2_k},
\end{align}
and
\begin{align}\label{eq:induction_hypothesis_contraction}
    \text{\textbf{Representation error}:}\;\  d(\hat \Phi_{k}, \Phi_\star) \leq \rho^{k N} d(\hat \Phi_{0}, \Phi_\star) + \frac{\Ccontract \sqrt{N} \log (HT)}{1-\sqrt{2}\rho^N \sqrt{H \tau_{k} \sigma_k^2}},
\end{align}

\begin{itemize}
    \item \textbf{Controlling the least-square estimation error:} To control the estimation error along the epochs we first show that after the first epoch, the estimation error is sufficiently small. To do so, we leverage the epoch-wise bounds on the least squares error, and on the representation error. In particular, note that the contribution of the representation to the least squares error is given by $\frac{\beta_1H^{2/5} d^2(\hat \Phi_k, \Phi_\star)}{\sigma^2_k}$. This may be bounded using the representation error bound in terms of the initial representation. In particular, 
    \begin{align*}
        \frac{\beta_1 H^{2/5} d^2(\hat \Phi_k, \Phi_\star)}{\sigma^2_k} &\leq 2 \frac{\beta_1 H^{2/5}\rho^{2 k N } d(\hat\Phi_0, \Phi_\star)^2}{\sigma^2_k} + 2\frac{\beta_1 H^{2/5} \log^2(HT) \Ccontract^2 N }{(1-\sqrt{2}\rho^N)^2 H \tau_k \sigma^2_k}.
    \end{align*}
    Using the lower bound on $\sigma_k^2$ of $\rho^{(k-1)N} d(\hat \Phi_0, \Phi_\star)$, for the first term, and $\sigma_k^2 \geq \tau_k^{-1/4} H^{-1/5}$ for the second term, we find that
    \begin{align*}
        \frac{\beta_1 H^{2/5} d^2(\hat \Phi_k, \Phi_\star)}{\sigma^2_k} &\leq 2 \beta_1 H^{2/5}  \rho^{2 N + (k-1) N } d(\hat\Phi_0, \Phi_\star)+ 2\frac{\beta_1 H^{2/5}  \log^2(HT) \Ccontract^2 N }{(1-\sqrt{2}\rho^N)^2 H^{5/6} \tau_k^{3/4}} \\
        &\leq 2 \beta_1 H^{2/5} d(\hat\Phi_0, \Phi_\star)+ 2\frac{\beta_1 \log^2(HT) \Ccontract^2 N }{(1-\sqrt{2}\rho^N)^2 H^{2/5} \tau_1^{3/4}}.
    \end{align*}
We have from Assumption~\ref{asmp: upper bound on representation error exp} that $2 \beta_1 H^{2/5} \rho^{2 N } d(\hat\Phi_0, \Phi_\star) \leq \frac{1}{4}\varepsilon^{\wedge}$. Similarly, $\tau_{\mathsf{warm}}$ may be selected such that $2\frac{\beta_1\log^2(HT) \Ccontract^2 N }{(1-\sqrt{2}\rho^N)^2 H^{2/5} \tau_1^{3/4}} \leq \frac{1}{4}\varepsilon^{\wedge}$. 
    Moreover, we can use a condition on the first epoch length such that $\tau_k \geq \tau_1 \geq c\paren{\sigma^2 (\pzeromax) \frac{\sqrt{\dtheta \du}}{\epsmin} \log (HT)}^2$, for a sufficiently large constant  $c$, along with the condition on the exploratory sequence $\sigma^2_k \geq \sqrt{\frac{\du\dtheta}{\tau_k}}$ to 
    show that $C_{\text{est,1}}\frac{\dtheta \sigma^2 (\pzeromax) }{\tau_k \sigma_k^2} \log\paren{HT} \leq \frac{1}{2} \epsmin$. Combining these facts implies that $\norm{\bmat{\hat A^{(h)}_k & \hat B^{(h)}_k} - \bmat{A_\star^{(h)} & B_\star^{(h)}}}_F^2 \leq \epsmin \leq \varepsilon^{(h)}$. Therefore, the conditions of Lemma \ref{lem: CE closeness} are satisfied and we may write 
    \begin{align*}
        \tau_{\text{ls}}(\hat K^{(h)}_{k+1}, x_b^2\log T, \frac{1}{3}T^{-3}) &\leq  2\tau_{\text{ls}}(K_\star^{(h)}, x_b^2\log T, \frac{1}{3} T^{-3}) \text{ and }\norm{P^{(h)}_{\hat K_{k+1}}}  \leq 1.05 (\pzeromax) \leq 2 (\pzeromax).
\end{align*}
where the first is true since the lower bound on $\tau_{\text{ls}}$ scales with $\|P^{(h)}_K\|$. Therefore, by selecting the initial epoch lengh according to $\tau_1 \geq c \tau_{\text{ls}}(K_\star^{(h)}, x_b^2 \log T, \frac{1}{2}T^{-3}),$ for a sufficiently large constant $c$, we can use Theorem \ref{thm: ls estimation error} to obtain, with probability $1 - \frac{1}{3}T^{-3}$, for all tasks $h \in [H]$, the following
\begin{align*}
       &\norm{\bmat{\hat A^{(h)}_{k+1} & \hat B^{(h)}_{k+1}} - \bmat{A^{(h)}_\star & B^{(h)}_\star}}_F^2 \lesssim \frac{\dtheta \sigma^2 (\pzeromax)}{\tau_{k+1} \sigma_{k+1}^2} \log\paren{HT} \\
        &\quad+ \paren{1+ \frac{\sigma^4 (\pzeromax)^9 (\psibmax) \paren{\dx+\du + \log\paren{HT}}}{\tau_{k+1} \sigma_{k+1}^4 } }  \frac{(\pzeromax)^3 (\psibmax)^2 d(\hat \Phi_{k+1}, \Phi_\star)^2 (\thetamax)^2}{\sigma_{k+1}^2},
\end{align*}
where we can use $\norm{P^{(h)}_{\hat K_{k+1}}}  \leq  2 (\pzeromax)$ and control the minimum eigenvalue of the input-state covariance matrix as follows
\begin{align*}
        \underset{h=1,\dots, H}{\min} \lambda_{\min}(\hat \Phi^\top_{k+1} \paren{\bar \Sigma^{\tau_{k+1}}(\hat K^{(h)}_{k+1},\sigma_{k+1},x^{(h)}_{k+1}) \otimes I_{\dx}} \hat \Phi_{k+1}) \geq  \frac{\sigma_{k+1}^2}{8 (\pzeromax)},
\end{align*}
which implies that from the condition $\sigma_{k+1}^2 \geq \tau_{k+1}^{-1/4} H^{-1/6}$ and the definition of $\beta_1$, we obtain
\begin{align*}
       &\norm{\bmat{\hat A^{(h)}_{k+1} & \hat B^{(h)}_{k+1}} - \bmat{A^{(h)}_\star & B^{(h)}_\star}}_F^2 \leq C_{\text{est,1}}\frac{\dtheta \sigma^2 (\pzeromax) }{\tau_{k+1} \sigma_{k+1}^2} \log\paren{HT} + \frac{\beta_1 H^{2/5} d^2(\hat \Phi_{k+1}, \Phi_\star)}{\sigma^2_{k+1}}.
    \end{align*}

Therefore, we proved that since $\mathcal{E}_{\text{ls},k}$ holds under high probability, then  $\mathcal{E}_{\text{ls},k}$ also holds under probability $1 -\frac{1}{3}T^{-3}$. By union bounding for all the epochs we have $\calE_{\mathsf{est},1} \subseteq  \calE_{\text{ls},1} \cap \dots \cap \calE_{\text{ls}, k_{\fin}}$ holds under probability of at least $1 - \frac{1}{3}T^{-2}.$
\vspace{0.2cm}

\item \textbf{The bounds on $\|x^{(h)}_t\|^2$ for $t =\{\tau_k+1,\ldots,\tau_{k+1}\}$ and $K^{(h)}_0$ are not violated:} By following our inductive hypothesis, we have  
\begin{align*}
    \norm{x^{(h)}_{\tau_k}} \leq 16 (\pzeromax)^{3/2} (\psibmax) \max_{1\leq t \leq T} \norm{\bmat{w^{(h)}_t \\ g^{(h)}_t}},
\end{align*}
which combined with $\norm{P^{(h)}_{\hat K_{k+1}}}  \leq  2 (\pzeromax)$ and $\tau_1 \geq c\frac{\log \frac{1}{\pstarmin}}{\log\paren{1- \frac{1}{ \pstarmin}}}$, for a sufficiently large constant $c$, we can exploit Lemma \ref{lem: state rollout bounds} to write
\begin{align}\label{eq:state_norm_first_epoch_1}
        \norm{x^{(h)}_t} \leq 40 (\pzeromax)^2(\psibmax) \max_{1\leq t \leq T} \norm{\bmat{w^{(h)}_t \\ g^{(h)}_t}}, \;\ \forall t = \{\tau_k+1,\ldots,\tau_{k+1}\},
\end{align}
and by using \eqref{eq:noise_bound_mt} in \eqref{eq:state_norm_first_epoch_1}, the state bound satisfies $\norm{x^{(h)}_t}^2 \leq x^2_b\log T$ with probability $1- \frac{1}{3}T^{-2}$, for all tasks $h \in [H]$. Moreover, the controller bound is satisfied since $\norm{\hat K^{(h)}_{k+1}}^2 \leq \norm{P_{\hat K_{k+1}}} \leq 2 (\pzeromax)$, which implies that $\norm{\hat K^{(h)}_{k+1}} \leq K_b$. Therefore, $\mathcal{E}_{\text{bound},k+1}$ holds under probability $1 - \frac{1}{3}T^{-2}$, which implies that  $\calE_{\mathsf{bound}}$ holds under probability of at least $1 - \frac{1}{3}T^{-2}.$

\vspace{0.2cm}
\item \textbf{Controlling the error in the learned representation:} Following our inductive hypothesis on the contraction of the learned representation from the previuos time step, we find that the conditions of Assumption~\ref{assumption: dfw burn-in} are met at the current time step for the current epoch with appropriate choice of $\tau_{\mathsf{warm}}$ which is polynomial in the system parameters stated in \Cref{thm: regret bound naive exploration}. Then we can use \Cref{thm: dfw bound formal} to obtain 
\begin{align}\label{eq:contraction_k+1}
        d\paren{\hat\Phi_{k+1}, \Phi_\star} &\leq \rho^N d\paren{\hat\Phi_{k}, \Phi_\star} + \frac{\Ccontract\sqrt{N}\log(HT)}{\sqrt{H\tau_{k+1}\sigma_{k+1}^2}},
\end{align}
with probability $1 - \frac{1}{3}T^{-3}$. Therefore, by applying  \eqref{eq:induction_hypothesis_contraction} to \eqref{eq:contraction_k+1} we have 
\begin{align*}
        d\paren{\hat\Phi_{k+1}, \Phi_\star} &\leq \rho^N \rho^{kN} d\paren{\hat\Phi_{0}, \Phi_\star} + \frac{\rho^{N}\Ccontract\sqrt{N}}{1-\sqrt{2}\rho^N}\frac{\log(HT)}{\sqrt{H\tau_{k}\sigma_{k}^2}} + \frac{\Ccontract\sqrt{N}\log(HT)}{\sqrt{H\tau_{k+1}\sigma_{k+1}^2}}\\
        &\stackrel{(i)}{\leq} \rho^{(k+1)N} d\paren{\hat\Phi_{0}, \Phi_\star} + \frac{\sqrt{2}\rho^N \Ccontract\sqrt{N}}{1-\sqrt{2}\rho^N}\frac{\log(HT)}{\sqrt{H\tau_{k+1}\sigma_{k+1}^2}} + \frac{\Ccontract \sqrt{N}\log(HT)}{\sqrt{H\tau_{k+1}\sigma_{k+1}^2}}\\
        &= \rho^{(k+1)N} d\paren{\hat\Phi_{0}, \Phi_\star} + \left(1 + \frac{\sqrt{2}\rho^N}{1-\sqrt{2}\rho^N}\right)\frac{\Ccontract \sqrt{N}\log(HT)}{\sqrt{H\tau_{k+1}\sigma_{k+1}^2}}\\
        &= \rho^{(k+1)N} d\paren{\hat\Phi_{0}, \Phi_\star} + \frac{\Ccontract \sqrt{N}}{1-\sqrt{2}\rho^N}\frac{\log(HT)}{\sqrt{H\tau_{k+1}\sigma_{k+1}^2}},
\end{align*}
where $(i)$ follows from the fact that $\tau_{k}\sigma^2_k \geq \frac{1}{2}\tau_{k+1}\sigma^2_{k+1}$. Therefore, we conclude that since $\mathcal{E}_{\text{c},k}$ holds under probability $1 - \frac{1}{3}T^{-3}$, then $\mathcal{E}_{\text{c},k+1}$ also holds under at least the same probability. Then, by union bounding for all the epochs, we have that $\calE_{\mathsf{cont}} \subseteq  \calE_{\text{c},1} \cap \dots \cap \calE_{\text{c}, k_{\fin}}$ holds under probability of at least $1 - \frac{1}{3}T^{-2}.$
\end{itemize}

We complete the proof by union bounding the events $\calE_{\mathsf{bound}}$, $\calE_{\mathsf{est},1}$, and $\calE_{\mathsf{cont}}$. We then have that $\mathcal{E}_{\text{success},1} \subseteq \calE_{\mathsf{bound}} \cap \calE_{\mathsf{est},1} \cap \calE_{\mathsf{cont}}$ holds under probability of at least $1- T^{-2}$.

\end{proof}
\end{lemma}

\subsection{High Probability Bound on Success Event 2 (Easy to identify parameters)}
\begin{lemma}
    \label{lem: success event bound 2}
    Running \Cref{alg: ce with exploration} with the arguments defined in \Cref{thm: regret bound no exploration}, the event $\calE_{\mathsf{success},2} $ holds with probability at least $1-T^{-2}$.
    \begin{proof}

Analogous to the probability of success event $\mathcal{E}_{\text{success}}$, we show that $\mathcal{E}_{\text{success},2}$ holds with probability $1 - T^{-2}$ by induction. To do so, we show, with high probability, that for every epoch $k \in \left[k_{\text {fin }}\right]$, Algorithm \ref{alg: ce with exploration} does not abort, i.e., the state and controller bounds are satisfied, the least-square estimation error is maintained small and scales according to the bound in $\mathcal{E}_{\text {est},2}$, and the learned common representation contracts towards its optimal as in $\mathcal{E}_{\text {cont}}$. We begin our analysis by studying the first epoch.

\vspace{0.5cm}

\noindent \textbf{Base case:} We consider the first epoch $k=1$ as the base case of the induction approach. For convenience we assume that $x^{(h)}_1 = 0$, for all tasks $h \in [H]$. However, it is worth noting that our proofs can be readily extended to bounded non-zero initial states.

\begin{itemize}
    \item \textbf{The bounds on $\|x^{(h)}_t\|^2$ for $t =\{0,1,\ldots,\tau_1\}$ and $K^{(h)}_0$ are not violated:} We begin our the analysis, by showing with high probability that the state and controller bounds are not violated. In order to ensure that the bounds on the state and controller are not violated, we first bound the worst-case behavior of the process and exploratory noises. For this purpose, we use Lemma \ref{lem: noise bound} to obtain
      \begin{align}\label{eq:noise_bound_mt_success_2}
         \max_{1\leq t \leq T} \norm{\bmat{w^{(h)}_t \\ g^{(h)}_t}} \leq 4\sigma\sqrt{3 (\dx + \du) \log (3HT)}.  
    \end{align}
with probability $1 - \frac{1}{3}T^{-2}$, for all tasks $h \in [H]$. Therefore, since the initial state satisfy 
$$\norm{x^{(h)}_1} \leq 16 (\pzeromax)^{3/2} (\psibmax) \max_{1\leq t \leq T} \norm{\bmat{w^{(h)}_t \\ g^{(h)}_t}},$$ 
and the initial epoch length can be selected such that $\tau_1 \geq \frac{c\log \frac{1}{8 \sqrt{\pstarmin}}}{\log\paren{1- \frac{1}{2 \pstarmin}}}$, for a sufficiently large constant $c$, respectively. We use Lemma \ref{lem: state rollout bounds} to write 
\begin{align}\label{eq:state_norm_first_epoch_success_2}
        \norm{x^{(h)}_t} \leq 40 (\pzeromax)^2 (\psibmax) \max_{1\leq t \leq T} \norm{\bmat{w^{(h)}_t \\ g^{(h)}_t}}, \;\ \forall t = \{0,1,\ldots,\tau_1\}
\end{align}

Therefore, by using \eqref{eq:noise_bound_mt_success_2} in \eqref{eq:state_norm_first_epoch_success_2} we have
\begin{align*}
        \norm{x^{(h)}_t}^2 \leq 76800 (\pzeromax)^4 (\psibmax)^2 \sigma^2 (\dx+ \du)\log (3HT), \;\ \forall t = \{0,1,\ldots,\tau_1\}
\end{align*}
with probability $1- \frac{1}{3}T^{-2}$, for all tasks $h \in [H]$. This implies that the state bound is satisfied, i.e., $\norm{x^{(h)}_t}^2 \leq x^2_b\log T$. On the other hand, for the controller bound, we note that $\|K^{(h)}_0\|^2 \leq \pzeromax \leq 2\pzeromax$, which implies that $\|K^{(h)}_0\| \leq K_b$. Therefore, $\mathcal{E}_{\text{bound},1}$ (i.e., the event where the state and controller bounds are satisfied at the first epoch) holds with probability $1-\frac{1}{3}T^{-2}$.
\vspace{0.2cm}

\item \textbf{Controlling the least-square estimation error:} To control the estimation error at the first epoch, we can use Theorem \ref{thm: ls estimation error}. In addition, a condition $\tau_{\mathsf{warm\_ up }} \geq \sigma^4 (\pzeromax)^3(\psibmax)^2  (\dx+\du)$ implies that $\tau_1 \geq c \tau_{\text{ls}}(K_0, 0, \frac{1}{3}T^{-3})$, for a sufficiently large constant $c$. Then,  from Theorem \ref{thm: ls estimation error}, we have

\begin{align*}
        &\norm{\bmat{\hat A^{(h)}_1 & \hat B^{(h)}_1} - \bmat{A^{(h)}_\star & B^{(h)}_\star}}_F^2 \lesssim \frac{\dtheta \sigma^2 \log (HT)  }{\tau_1   \underset{h=1,\dots, H}{\min}\lambda_{\min}(\hat \Phi^\top_1\paren{\bar \Sigma^{\tau_1}_h(K^{(h)}_0,\sigma^2_1,0) \otimes I_{\dx}} \hat \Phi_1)}  \\
        &+ \paren{1+\frac{\sigma^4 (\pzeromax)^7 (\psibmax)^6 \paren{\dx+\du + \log (HT)}}{\tau_1 \underset{h=1,\dots, H}{\min}\lambda_{\min}(\hat \Phi^\top_1 \paren{\bar \Sigma^{\tau_1}_h(K^{(h)}_0,\sigma^2_1,0) \otimes I_{\dx}} \hat \Phi_1)^2 } } \frac{(\pzeromax)^2 (\psibmax)^2 d(\hat \Phi_1, \Phi_\star)^2 (\thetamax)^2}{ \underset{h=1,\dots, H}{\min}\lambda_{\min}(\hat \Phi^\top_1 \paren{\bar \Sigma^{\tau_1}_h(K^{(h)}_0,\sigma^2_1,0) \otimes I_{\dx}} \hat \Phi_1)} .
    \end{align*}
with probability $1 - \frac{1}{3}T^{-3}$, for all tasks $h \in [H]$. The rate of the decay in the estimation error is controlled by the minimum eigenvalue of the input-state covariance matrix. The main difference between this proof to the one for $\mathcal{E}_{\text{success},1}$ is on the lower bound of minimum eigenvalue of the input-state covariance matrix. Here, we note that for any unit vector $v$,
\begin{equation*}
    \begin{aligned}
        \underset{h=1,\dots, H}{\min} v^\top \hat \Phi^\top_1 \paren{\bar \Sigma^{\tau_1}_h(K^{(h)}_0,\sigma^2_1,0) \otimes I_{\dx}} \hat \Phi_1 v &\geq \underset{h=1,\dots, H}{\min} \frac{1}{2}  v \hat \Phi^\top_1 \paren{\bmat{I \\ K^{(h)}_0} \bmat{I \\ K^{(h)}_0}^\top \otimes I_{\dx}} \hat \Phi_1 v \\
        &=\underset{h=1,\dots, H}{\min} \frac{1}{2} \norm{G \hat \Phi_1^\top v}^2
\end{aligned}
\end{equation*}
where $G \triangleq \paren{\bmat{I \\ K^{(h)}_0} \bmat{I \\ K^{(h)}_0}^\top \otimes I_{\dx}}^{1/2}$. Using the fact that $\Phi_\star \Phi_\star^\top + \Phi_{\star, \perp} \Phi_{\star,\perp}^\top$, we note that for any unit vector $v$
\begin{equation}
\label{eq: cov lower bound w/ rep error}
\begin{aligned}
\norm{G \hat \Phi_1^\top v}^2 &= \norm{G (\Phi_\star \Phi_\star^\top + \Phi_{\star, \perp} \Phi_{\star,\perp}^\top) \hat \Phi_1^\top v}^2 \\&\geq \frac{1}{2} \norm{G \Phi_\star \Phi_\star^\top  \hat \Phi_1^\top v}^2 - \norm{G \Phi_{\star, \perp} \Phi_{\star,\perp}^\top \hat \Phi_1^\top v}^2 \\
    &\geq \frac{1}{2}  \lambda_{\min}(\Phi_\star^\top G^2 \Phi_\star) \norm{\Phi_\star^\top \hat \Phi_1 v}^2 - \norm{G}^2 d(\hat \Phi_1, \Phi_\star)^2 ,
\end{aligned}
\end{equation}
where $\tilde v = \frac{\Phi_\star^\top \hat \Phi_1 v}{\norm{\Phi_\star^\top \hat \Phi_1 v}}$. By Assumption~\ref{asmpt: persisent excitation}, $\lambda_{\min}(\Phi_\star^\top G^2 \Phi_\star) \geq \alpha^2$. Additionally, note that
\begin{align*}
    \norm{\Phi_\star^\top \hat \Phi_1 v}^2  &= v^\top \hat \Phi_1^\top \Phi_\star \Phi_\star^\top \hat\Phi_1 v \\
    &= v^\top \hat\Phi_1^\top \hat \Phi_1 v - v^\top \hat \Phi_1^\top \Phi_{\star,\perp} \Phi_{\star,\perp}^\top \hat\Phi_1 v \geq 1 - d(\hat \Phi_1, \Phi_\star)^2. 
\end{align*}
By the assumptions on the initial representation, and the size of the initial epoch, we ensure that $d(\hat\Phi_1, \Phi_\star)^2 \leq \min\curly{\frac{1}{2}, \frac{\alpha^2}{64 (1 + P_0^\wedge)}}$. Therefore, by combining the above sequence of inequalities, we have that 
\begin{equation*}
   \underset{h=1,\dots, H}{\min} \lambda_{\min}(\hat \Phi^\top_1 \paren{\bar \Sigma^{\tau_1}_h(K^{(h)},0,0) \otimes I_{\dx}} \hat \Phi_1) \geq \frac{\alpha^2}{16},
\end{equation*}
which implies that
\begin{align*}
        &\norm{\bmat{\hat A^{(h)}_1 & \hat B^{(h)}_1} - \bmat{A^{(h)}_\star & B^{(h)}_\star}}_F^2 \lesssim \frac{\dtheta \sigma^2 }{\tau_1 \alpha^2} \log\paren{HT} \\
        &\quad+ \paren{1+ \frac{\sigma^4 (\pzeromax)^7 (\psibmax)^6 \paren{\dx+\du + \log\paren{HT}}}{\tau_1 \alpha^4 } } \frac{(\pzeromax)^2 (\psibmax)^2 d(\hat \Phi_1, \Phi_\star)^2 (\thetamax)^2}{\alpha^2}.
    \end{align*}
and by using a condition on the initial epoch length $\tau_1 \geq \frac{c\sigma^2 \dtheta \log(HT)}{2 \epsmin \alpha^2}$, for a sufficiently large constant $c$, we have 
\begin{align*}
        &\norm{\bmat{\hat A^{(h)}_1 & \hat B^{(h)}_1} - \bmat{A^{(h)}_\star & B^{(h)}_\star}}_F^2 \leq C_{\mathsf{est},2} \frac{\sigma^2 \dtheta \log (HT)}{\tau_1 \alpha^2} + \beta_2 d(\hat \Phi_1, \Phi_\star)^2,
    \end{align*}
where $\beta_2 \triangleq C_{\mathsf{bias},2} \frac{\sigma^2\epsmin (\psibmax)^8 (\thetamax)^2 (\dx+\du)}{\dtheta \min\curly{\alpha^2, \alpha^4}}$. Then, by defining the event $\mathcal{E}_{\text{ls},1}$ where the above estimation error bound holds for the first epoch, we have that  $\mathcal{E}_{\text{ls},1}$ holds with probability $1 - \frac{1}{3}T^{-3}$ for all tasks $h \in [H]$. 

\vspace{0.2cm}
\item \textbf{Controlling the contraction in the learned representation:} For the first epoch, we initialize the representation as $\hat{\Phi}_0$. Then, Algorithm \ref{alg: ce with exploration} play $K^{(h)}_0$ for all tasks $h \in [H]$ to collect a multi-task dataset that is leveraged to update the representation $\hat\Phi_1$ with $N$ iterations of Algorithm \ref{alg: dfw}. By the bound on the initial representation error from Assumption~\ref{asmp: upper bound on representation error no exp} and the length of the initial epoch determined by $\tau_{\mathsf{warm\,up}}$, we can use \Cref{thm: dfw bound formal} to obtain with probability $1 -\frac{1}{3}T^{-3}$ that
\begin{align*}
        d\paren{\hat\Phi_1, \Phi_\star} &\leq \rho^N d\paren{\hat\Phi_0, \Phi_\star} + \frac{C_{\mathsf{contract}}\sqrt{N}\log(HT)}{\sqrt{H\tau_1\sigma_1^2}}.
\end{align*}
Denote this event by $\mathcal{E}_{\text{c},1}$.
\end{itemize}
\vspace{0.2cm}
\noindent \textbf{Induction step:} We now use an induction step with the following inductive hypothesis:

\begin{align}\label{eq:inductive_hypotesis_state_control_bounds_success_2}
   \text{\textbf{Bounded state}:}\;\ \norm{x^{(h)}_{\tau_k}} \leq 16 (\pzeromax)^{3/2} (\psibmax) \max_{1\leq t \leq T} \norm{\bmat{w^{(h)}_t \\ g^{(h)}_t}},
\end{align}

\begin{align}\label{eq:inductive_hypotesis_estimation_error_success_2}
       \text{\textbf{Least-square error}:} \norm{\bmat{\hat A^{(h)}_k  \hat B^{(h)}_k} - \bmat{A^{(h)}_\star  B^{(h)}_\star}}_F^2 \leq C_{\mathsf{est},2} \frac{\sigma^2 \dtheta \log (HT)}{\tau_k \alpha^2} + \beta_2 d(\hat \Phi_k, \Phi_\star)^2,
\end{align}
and
\begin{align}\label{eq:induction_hypothesis_contraction_success_2}
    \text{\textbf{Representation error}:}\;\ d\paren{\hat\Phi_k, \Phi_\star} &\leq \rho^{kN} d\paren{\hat\Phi_0, \Phi_\star} + \frac{C_{\mathsf{contract}}\sqrt{N}}{1-\sqrt{2}\rho^N}\frac{\log(HT)}{\sqrt{H\tau_k\sigma^2_k}},
\end{align}

\begin{itemize}
    \item \textbf{Controlling the least-square estimation error:} To control the estimation error throughout the epochs, we first note that $\sigma_k^2 \geq \tau_k^{-1/2} H^{-1/2}$, we have that $\tau_k \sigma_k^2 \geq \tau_k^{1/2} H^{-1/2} \geq \tau_1^{1/2} H^{-1/2}$. Using the condition $\tau_1 \geq 8 \frac{C_{\mathsf{contract}}^4 N^2 \log^4(HT)}{(1-\sqrt{2}\rho^N)^4H^{1/2}d(\hat \Phi_{0}, \Phi_\star)^4}$, we then have $d\paren{\hat\Phi_k, \Phi_\star} \leq  d\paren{\hat\Phi_0, \Phi_\star}$ for all $k \in \left[k_{\text {fin }}\right]$. Moreover, we can select the first epoch length as $\tau_1 \geq \frac{c\sigma^2 \dtheta \log(HT)}{2 \epsmin \alpha^2}$, for a sufficiently large constant $c$, along with the initial representation error $d(\hat \Phi_0, \Phi_\star) \leq \sqrt{\frac{\epsmin}{2 \beta_2}}$ to obtain $\norm{\bmat{\hat A^{(h)}_k & \hat B^{(h)}_k} - \bmat{A_\star^{(h)} & B_\star^{(h)}}}_F^2 \leq \epsmin \leq \varepsilon^{(h)}$. Therefore, the conditions of Lemma \ref{lem: CE closeness} are satisfied and we can write
    \begin{align*}
        \tau_{\text{ls}}(\hat K^{(h)}_{k+1}, x_b^2\log T, \frac{1}{3}T^{-3}) &\leq  2\tau_{\text{ls}}(K^{(h)}_\star, x_b^2\log T, \frac{1}{3} T^{-3}) \text{ and }  \norm{P^{(h)}_{\hat K_{k+1}}}  \leq 1.05 (\pzeromax) \leq 2 (\pzeromax).
\end{align*}
where the first is due to the fact that the lower bound on $\tau_{\text{ls}}$ scales with $\|P^{(h)}_K\|$. Therefore, by setting the first epoch length such that $\tau_1 \geq c \tau_{\text{ls}}(K^{(h)}_\star, x_b^2 \log T, \frac{1}{2}T^{-3}),$ for a sufficiently large universal constant $c$, we use Theorem \ref{thm: ls estimation error} to obtain
\begin{align*}
       &\norm{\bmat{\hat A^{(h)}_{k+1} & \hat B^{(h)}_{k+1}} - \bmat{A^{(h)}_\star & B^{(h)}_\star}}_F^2 \lesssim \frac{\dtheta \sigma^2 \log (HT)  }{\tau_{k+1}   \underset{h=1,\dots, H}{\min}\lambda_{\min}\left(\hat \Phi^\top_{k+1}\paren{\bar \Sigma^{\tau_{k+1}}(K^{(h)}_0,\sigma^2_{k+1},x^{(h)}_{\tau_{k+1}}) \otimes I_{\dx}} \hat \Phi_{k+1}\right)}  \\
        &+ \paren{1+\frac{\sigma^4 (\pzeromax)^7 (\psibmax)^6 \paren{\dx+\du + \log (HT)}}{\tau_{k+1} \underset{h=1,\dots, H}{\min}\lambda_{\min}\left(\hat \Phi^\top_{k+1} \paren{\bar \Sigma^{\tau_{k+1}}(K^{(h)}_0,\sigma^2_{k+1},x^{(h)}_{\tau_{k+1}}) \otimes I_{\dx}} \hat \Phi_{k+1}\right)^2 } }\\
        &\times \frac{(\pzeromax)^2 (\psibmax)^2 d(\hat \Phi_1, \Phi_\star)^2 (\thetamax)^2}{ \underset{h=1,\dots, H}{\min}\lambda_{\min}\left(\hat \Phi^\top_{k+1} \paren{\bar \Sigma^{\tau_{k+1}}(K^{(h)}_0,\sigma^2_{k+1},x^{(h)}_{\tau_{k+1}}) \otimes I_{\dx}} \hat \Phi_{k+1}\right)},
\end{align*}
with probability $1 - \frac{1}{3}T^{-3}$ for all tasks $h \in [H]$. In the above expression we also use $\norm{P^{(h)}_{\hat K_{k+1}}}  \leq  2 (\pzeromax)$. We control the minimum eigenvalue of the input-state covariance matrix as follows:
since $\lambda_{\min}(\Phi_\star^\top \paren{\bar \Sigma^{\tau_{k+1}} 
\paren{\bmat{I \\ \hat K_{k+1}^{(h)}} \bmat{I \\ \hat K_{k+1}^{(h)}}^\top \otimes I_{\dx}}} \Phi_\star)  =\inf_v \norm{\sum_{i=1}^{d_{\theta}} v_i (\Phi_{\star,i}^A + \Phi_{\star, i}^B \hat K^{(h)}_{k+1})}_F^2$ and 
\begin{align*}
       \underset{h=1,\dots, H}{\min} \norm{\sum_{i=1}^{\dtheta} v_i \paren{ \Phi^{A}_{\star,i} + \Phi^{B}_{\star, i} K^{(h)\star} } +\sum_{i=1}^{\dtheta} v_i \Phi^B_{\star, i} (\hat K^{(h)}_{k+1} - K^{(h)\star}) }_F &\geq \alpha - \norm{\hat K^{(h)}_{k+1} - K^{(h)\star}}\\
       &\geq \alpha - \frac{1}{6(\pzeromax)^{3/2}} \geq \frac{\alpha}{2}.
    \end{align*}
it holds from \eqref{eq: cov lower bound w/ rep error} along with the condition on the initial representation error that
\begin{align*}
    \lambda_{\min}\paren{\hat \Phi^\top_{k+1} \paren{\bar \Sigma^{\tau_1}_h(K^{(h)}_k+1,\sigma^2_{k+1},x_{\tau_{k+1}}^{(h)}) \otimes I_{\dx}} \hat \Phi_{k+1} } \geq \frac{\alpha^2}{64}. 
\end{align*}

\begin{align*}
       &\norm{\bmat{\hat A^{(h)}_{k+1} & \hat B^{(h)}_{k+1}} - \bmat{A^{(h)}_\star & B^{(h)}_\star}}_F^2 \leq C_{\mathsf{est},2} \frac{\sigma^2 \dtheta \log (HT)}{\tau_{k+1} \alpha^2} + \beta_2 d(\hat \Phi_{k+1}, \Phi_\star)^2,
\end{align*}

Therefore, since $\mathcal{E}_{\text{ls},k}$ holds with high probability, then  $\mathcal{E}_{\text{ls},k}$ also holds with probability $1 -\frac{1}{3}T^{-3}$. This implies that $\calE_{\mathsf{est},2} \subseteq  \calE_{\text{ls},1} \cap \dots \cap \calE_{\text{ls}, k_{\fin}}$ holds with probability of at least $1 - \frac{1}{3}T^{-2}$ for all tasks $h \in [H]$.
\vspace{0.2cm}

\item \textbf{The bounds on $\|x^{(h)}_t\|^2$ for $t =\{\tau_k + 1,\ldots,\tau_{k+1}\}$ and $K^{(h)}_0$ are not violated:} By following our inductive hypothesis, we have  
\begin{align*}
    \norm{x^{(h)}_{\tau_k}} \leq 16 (\pzeromax)^{3/2} (\psibmax) \max_{1\leq t \leq T} \norm{\bmat{w^{(h)}_t \\ g^{(h)}_t}},
\end{align*}
which combined with $\norm{P^{(h)}_{\hat K_{k+1}}}  \leq  2 (\pzeromax)$ and $\tau_1 \geq c\frac{\log \frac{1}{8 \sqrt{\pstarmin}}}{\log\paren{1- \frac{1}{2 \pstarmin}}}$, for a sufficiently large constant $c$, we can use Lemma \ref{lem: state rollout bounds} to write
\begin{align}\label{eq:state_norm_first_epoch_2}
        \norm{x^{(h)}_t} \leq 40 (\pzeromax)^2 (\psibmax) \max_{1\leq t \leq T} \norm{\bmat{w^{(h)}_t \\ g^{(h)}_t}}, \;\ \forall t = \{\tau_k+1,\ldots,\tau_{k+1}\},
\end{align}
and by using \eqref{eq:noise_bound_mt_success_2} in \eqref{eq:state_norm_first_epoch_2}, the state bound is satisfied , i.e., $\norm{x^{(h)}_t}^2 \leq x^2_b\log T$, with probability $1- \frac{1}{3}T^{-2}$, for all tasks $h \in [H]$. Moreover, the controller bound is verified since $\norm{\hat K^{(h)}_{k+1}}^2 \leq \norm{P_{\hat K^{(h)}_{k+1}}} \leq 2 (\pzeromax)$, which implies that $\norm{\hat K^{(h)}_{k+1}} \leq K_b$. Then, $\mathcal{E}_{\text{bound},k+1}$ holds with probability $1 - \frac{1}{3}T^{-2}$, which implies that $\calE_{\mathsf{bound}}$ holds with probability of at least $1 - \frac{1}{3}T^{-2}$, for all tasks $h \in [H]$.

\vspace{0.2cm}
\item \textbf{Controlling the error in the learned representation:} Following our inductive hypothesis on the contraction of the learned representation, we find that the conditions of Assumption~\ref{assumption: dfw burn-in} are met at the current epoch. We may therefore apply \Cref{thm: dfw bound formal} to obtain 

\begin{align}\label{eq:contraction_k+1_success_2}
        d\paren{\hat\Phi_{k+1}, \Phi_\star} &\leq \rho^N d\paren{\hat\Phi_{k}, \Phi_\star} + C_{\mathsf{contract}}\frac{\sqrt{N}\log(HT)}{\sqrt{H\tau_{k+1}\sigma_{k+1}^2}},
\end{align}
with probability $1 - \frac{1}{3}T^{-3}$. Therefore, by applying  \eqref{eq:induction_hypothesis_contraction_success_2} to \eqref{eq:contraction_k+1_success_2} we have 
\begin{align*}
        d\paren{\hat\Phi_{k+1}, \Phi_\star} &\leq \rho^N \rho^{kN} d\paren{\hat\Phi_{0}, \Phi_\star} + \frac{\rho^{N}C_{\mathsf{contract}}\sqrt{N}}{1-\sqrt{2}\rho^N}\frac{\log(HT)}{\sqrt{H\tau_{k}\sigma_{k}^2}} + C_{\mathsf{contract}}\frac{\sqrt{N}\log(HT)}{\sqrt{H\tau_{k+1}\sigma_{k+1}^2}}\\
        &\stackrel{(i)}{\leq} \rho^{(k+1)N} d\paren{\hat\Phi_{0}, \Phi_\star} + \frac{\sqrt{2}\rho^N C_{\mathsf{contract}}\sqrt{N}}{1-\sqrt{2}\rho^N}\frac{\log(HT)}{\sqrt{H\tau_{k+1}\sigma_{k+1}^2}} + C_{\mathsf{contract}}\frac{\sqrt{N}\log(HT)}{\sqrt{H\tau_{k+1}\sigma_{k+1}^2}}\\
        &= \rho^{(k+1)N} d\paren{\hat\Phi_{0}, \Phi_\star} + \left(1 + \frac{\sqrt{2}\rho^N}{1-\sqrt{2}\rho^N}\right)C_{\mathsf{contract}}\frac{\sqrt{N}\log(HT)}{\sqrt{H\tau_{k+1}\sigma_{k+1}^2}}\\
        &= \rho^{(k+1)N} d\paren{\hat\Phi_{0}, \Phi_\star} + \frac{C_{\mathsf{contract}}}{1-\sqrt{2}\rho^N}\frac{\sqrt{N}\log(HT)}{\sqrt{H\tau_{k+1}\sigma_{k+1}^2}},
\end{align*}
where $(i)$ follows from the fact that $\tau_{k}\sigma^2_k \geq \frac{1}{2}\tau_{k+1}\sigma^2_{k+1}$. Therefore, we conclude that since $\mathcal{E}_{\text{c},k}$ holds with probability $1 - \frac{1}{3}T^{-3}$, then $\mathcal{E}_{\text{c},k+1}$ also holds with at least the same probability. Then, by union bounding for all the epochs, we have that $\calE_{\mathsf{cont}} \subseteq  \calE_{\text{c},1} \cap \dots \cap \calE_{\text{c}, k_{\fin}}$ holds under probability of at least $1 - \frac{1}{3}T^{-2}.$
\end{itemize}

We complete the proof by union bounding the events $\calE_{\mathsf{bound}}$, $\calE_{\mathsf{est},1}$, and $\calE_{\mathsf{cont}}$. Then, we have that $\mathcal{E}_{\text{success},2} \subseteq \calE_{\mathsf{bound}} \cap \calE_{\mathsf{est},2} \cap \calE_{\mathsf{cont}}$ holds under probability of at least $1- T^{-2}$.

\end{proof}
\end{lemma}

\section{Synthesizing the regret bounds}
\label{s: synthesizing the bounds}  

We use the success events to decompose the expected regret as in \cite{cassel2020logarithmic}: 
$
    \E \brac{\calR_T^{(h)}} = R_1^{(h)} + R_2^{(h)}+ R_3^{(h)} - T\calJ(K_\star^{(h)}), 
$
where for $\calE_{\mathsf {success}} = \calE_{\mathsf{success},1}$ or $\calE_{\mathsf {success}} =\calE_{\mathsf{success}, 2}$,
\begin{align}
    \label{eq: regret decomposition}
    R_1^{(h)} = \E \brac{\mathbf{1}(\calE_{\mathsf{success}}) \sum_{k=2}^{k_\fin} J_k^{(h)}}, \quad R_2^{(h)} =  \E \brac{\mathbf{1}(\calE_{\mathsf{success}}^c) \sum_{t=\tau_1 + 1}^{T} c_t^{(h)}}, \quad  \mbox{ and } \quad R_3^{(h)} = \E \brac{\sum_{t=1}^{\tau_1} c_t^{(h)}}.
\end{align}
Here, $J_k^{(h)} = \sum_{t=\tau_{k}}^{\tau_{k+1}-1} c_t^{(h)}$ is the epoch cost 
and $c_t^{(h)}= x_t^{(h) \top} Q x_t^{(h)} + u_t^{(h) \top} R u_t^{(h)}$ is the stage cost. 

The terms $R_2^{(h)}$ and $R_3^{(h)}$ may be bounded directly by invoking Lemmas 20 and 22 of \cite{lee2023nonasymptotic} along with the high probability bounds of \Cref{lem: success event bound 1} and \Cref{lem: success event bound 2}. It therefore remains to bound $R_1^{(h)} - T \calJ^{(h)}(K_\star^{(h)})$. This is done for the settings of \Cref{thm: regret bound naive exploration} and \Cref{thm: regret bound no exploration} in the following two lemmas. 

\begin{lemma}\label{lem: success event regret bound 1}
In the setting of \Cref{thm: regret bound naive exploration}, we have
\begin{align*}
    R_1^{(h)} - T \calJ^{(h)}(K_\star^{(h)}) &\leq  \texttt{poly}\paren{\dx, \du,\pzeromax, \psibmax, \tau_{\mathsf{warm\,up}}, x_b, \frac{1}{1-\sqrt{2}\rho^N}, d(\hat \Phi_0,\Phi_\star)} \log^2 T \\
        &+ \texttt{poly}\paren{\pzeromax, \psibmax, \sigma}\sqrt{\dtheta \du} \sqrt{T} \log T \\
        &+ \texttt{poly}\paren{\dx, \du, \dtheta, \pzeromax, \psibmax, \thetamax, \sigma, \frac{1}{1-\sqrt{2}\rho^N}, N}\frac{T^{3/4}}{H^{1/5}} \log^2(TH).
\end{align*}
\end{lemma}
\begin{proof}
We may invoke Lemma 22 of \cite{lee2023nonasymptotic} to show that 
    \begin{equation}
        \label{eq: R_1 bound}
        \begin{aligned}
        R_1^{(h)} - T \calJ^{(h)}(K_\star^{(h)})& \leq \sum_{k=2}^{k_{\fin}} 142 \norm{P_\star^{(h)}}^8 (\tau_k - \tau_{k-1}) \E\brac{\mathbf{1}\brac{E_k^{(h)}} \norm{\bmat{\hat A_{k-1}^{(h)} & \hat B_{k-1}^{(h)}} - \bmat{A_\star^{(h)} & B_\star^{(h)}}}^2} \\&+ (\tau_k - \tau_{k-1}) \calJ^{(h)}(K_\star^{(h)})  
        + 4(\tau_k - \tau_{k-1}) \du \norm{P_{\star}^{(h)}} \sigma_k^2 \Psi_{B_\star^{(h)}}^2 + 2 x_b \log T \norm{P_\star^{(h)}}. 
    \end{aligned}
    \end{equation}
    
where 
\begin{align*}
   E_k^{(h)} &=  \curly{d(\hat \Phi_{k-1}, \Phi_\star) \leq \rho^{(k-1) N} d(\hat \Phi_{0}, \Phi_\star) + \frac{C_{\mathsf{contract} }\sqrt{\pzeromax} \sqrt{N} \log (HT)}{(1-\sqrt{2}\rho^N) \sqrt{H \tau_{k} \sigma_k^2} }} \\
   &\cap \curly{\norm{\bmat{\hat A_{k-1}^{(h)} & \hat B_{k-1}^{(h)}} - \bmat{A_\star^{(h)} & B_\star^{(h)}}}_F^2 \leq C_{\mathsf{est},1} \frac{\sigma^2 \dtheta \norm{P_{K_0}^{(h)}}}{\tau_{k-1} \sigma_{k-1}^2} \log T + \frac{\beta_1 H^{2/5} d(\hat \Phi_{k-1}, \Phi_\star)^2 }{\sigma_k^2}}
\end{align*}
is the event bounding the norm of the dynamics error in terms of the misspecification as well as the misspecification in terms of the amount of data. Under the event $E_k^{(h)}$, we have
\begin{align*}
    &\E\brac{\mathbf{1}\brac{E_k^{(h)}}  \norm{\bmat{\hat A_{k-1}^{(h)} & \hat B_{k-1}^{(h)}} - \bmat{A_\star^{ (h)} & B_\star{ (h)}}}^2} \\
    &\leq C_{\mathsf{est},1} \frac{\sigma^2 \dtheta \norm{P_{K_0}^{(h)}}}{\tau_{k-1} \sigma_{k-1}^2} \log T + \frac{2 \beta_1 \rho^{2(k-1)N}  d(\hat \Phi_0, \Phi_\star)^2}{\sigma_{k-1}^2}  + 2 \beta_1 \frac{C_{\mathsf{contract} }^2 \pzeromax N \log^2 (HT)}{(1-\sqrt{2}\rho^N)^2 H^{3/5} \tau_{k-1} \sigma_{k-1}^4}.
\end{align*}
Substituting the above inequality into \eqref{eq: R_1 bound}, we find
        \begin{align*}
            &R_1^{(h)} - T \calJ^{(h)}(K_\star^{(h)}) \\
            &\lesssim \sum_{k=2}^{k_{\fin}} \left(\norm{P_\star^{(h)}}^8  \tau_{k-1}\paren{\frac{\sigma^2 \dtheta \norm{P_{K_0}^{(h)}}}{\tau_{k-1} \sigma_{k-1}^2} \log T + \frac{\beta_1 \rho^{2(k-1)N}  d(\hat \Phi_0, \Phi_\star)^2}{\sigma_{k-1}^2}  +  \beta_1 \frac{C_{\mathsf{contract} }^2 \pzeromax N \log^2 (HT)}{(1-\sqrt{2}\rho^N)^2 H^{3/5} \tau_{k-1} \sigma_{k-1}^4}}\right) \\&
        + \tau_{k-1} \du \norm{P_{\star}^{(h)}} \sigma_k^2 \Psi_{B_\star^{(h)}}^2 +  x_b \log T \norm{P_\star^{(h)}} .
    \end{align*}
    Substituting in our choice of $\sigma_k^2$ from \eqref{eq: exploration, hard to learn}, we find
    \begin{align*}
        &R_1^{(h)} - T \calJ^{(h)}(K_\star^{(h)}) \\
        &\lesssim \sum_{k=2}^{k_{\fin}} \sigma^2 \sqrt{d_{\theta} \du} \norm{P_{K_0^{(h)}}^{(h)}}^9 \Psi_{B_\star^{(h)}}^2 \sqrt{\tau_{k-1}} \log T + \frac{ \paren{\norm{P_{K_0^{(h)}}^{(h)}}^8\beta_1 C_{\mathsf{contract}}^2 \pzeromax N \log^2(HT) + \du \norm{P_{K_0^{(h)}}^{(h)}} \Psi_{B_\star^{(h)}}^2}  }{(1-\sqrt{2}\rho^N) H^{1/6} } \tau_{k-1}^{3/4}\\
        & +  \paren{\beta_1 \norm{P_{K_0^{(h)}}^{(h)}}^8 + \du \norm{P_{K_0^{(h)}}^{(h)}} \Psi_{B_\star^{(h)}}^2} \tau_{k-1} \rho^{(k-1) N}d(\Phi_0, \Phi_\star)+ x_b \log T \norm{P_\star^{(h)}} \\
        & \lesssim \sigma^2 \sqrt{d_{\theta} \du} \norm{P_{K_0^{(h)}}^{(h)}}^9 \Psi_{B_\star^{(h)}}^2  \sqrt{T} \log T + \frac{ \paren{\beta_1 \norm{P_{K_0^{(h)}}^{(h)}}^8 C_{\mathsf{contract}}^2 \pzeromax N \log^2(HT) + \du \norm{P_{K_0^{(h)}}^{(h)}} \Psi_{B_\star^{(h)}}^2}  }{(1-\sqrt{2}\rho^N)  } \frac{T^{3/4}}{H^{1/6}} \\
        & +\frac{\paren{\beta_1 \norm{P_{K_0^{(h)}}^{(h)}}^8 + \du \norm{P_{K_0^{(h)}}^{(h)}} \Psi_{B_\star^{(h)}}^2} \tau_1}{1-\sqrt{2}\rho^N} d(\Phi_0, \Phi_\star) + x_b \log^2 T \norm{P_{K_0^{(h)}}^{(h)}}.
        \end{align*}
        The result now follows by substituting in the definition of $\tau_1$ from \Cref{thm: regret bound naive exploration}, of $\beta_1$ from Assumption~\ref{asmp: upper bound on representation error exp}, and of $C_{\mathsf{contract}}$ from \Cref{thm: dfw bound formal}.
\end{proof}

\begin{lemma}
    \label{lem: success event regret bound 2}
    In the setting of \Cref{thm: regret bound no exploration}, we have
    \begin{align*}
        R_1^{(h)} 
         &\leq \texttt{poly}\paren{\sigma, \dtheta, \du, \dx, \frac{1}{\alpha}, \frac{1}{1-\sqrt{2}\rho^N}, \psibmax, \pzeromax, x_b, \tau_{\mathsf{warm\,up}}, d(\hat \Phi_0, \Phi_\star)} \log^2 T \\
         & + \texttt{poly}\paren{\sigma, \dtheta, \du, \dx, \frac{1}{\alpha}, \frac{1}{1-\sqrt{2}\rho^N}, \psibmax, \pzeromax, x_b, N} \frac{\sqrt{T}}{\sqrt{H}} \log^2(TH).
    \end{align*}
\end{lemma}
\begin{proof}
    We again invoke Lemma 22 of \cite{lee2023nonasymptotic} to show that \eqref{eq: R_1 bound} holds in this setting, where the event bounding the norm of the dynamics error is now given by 
    \begin{align*}
     E_k^{(h)} &=  \curly{d(\hat \Phi_{k-1}, \Phi_\star) \leq \rho^{(k-1) N} d(\hat \Phi_{0}, \Phi_\star) + \frac{C_{\mathsf{contract} } \sqrt{\pzeromax} \sqrt{N} \log (HT)}{(1-\sqrt{2}\rho^N) \sqrt{H \tau_{k} \sigma_k^2} }} \\
    &\cap \curly{
   \norm{\bmat{\hat A_k^{(h)} & \hat B_k^{(h)}} - \bmat{A_\star^{(h)} & B_\star^{(h)}}}_F^2 \leq C_{\mathsf{est},2} \frac{\sigma^2  \dtheta   }{ \tau_{k} \alpha^2 }\log T+ \beta_2 d(\hat \Phi_k, \Phi_\star)^2
    }.
\end{align*}
    Under this event, we have 
        \begin{align*}
            &\E\brac{\mathbf{1}\brac{E_k^{(h)}}  \norm{\bmat{\hat A_{k-1}^{(h)} & \hat B_{k-1}^{(h)}} - \bmat{A_\star^{ (h)} & B_\star{ (h)}}}^2} \\
            &\leq C_{\mathsf{est},2} \frac{\sigma^2  \dtheta   }{ \tau_{k-1} \alpha^2 }\log T + \frac{2 \beta_2 \rho^{2(k-1)N}  d(\hat \Phi_0, \Phi_\star)^2}{\sigma_{k-1}^2}  + 2 \beta_2 \frac{C_{\mathsf{contract}}^2 \pzeromax N \log^2 (HT)}{(1-\sqrt{2}\rho^N)^2 H \tau_{k-1} \sigma_{k-1}^2}.
        \end{align*}
    Substituting the above inequality into \eqref{eq: R_1 bound}, we have 
   \begin{align*}
        &R_1^{(h)} - T \calJ^{(h)}(K_\star^{(h)}) \\&\lesssim\sum_{k=2}^{k_{\fin}}  \norm{P_\star^{(h)}}^8 \tau_{k-1} \paren{ \frac{\sigma^2  \dtheta   }{ \tau_{k-1} \alpha^2 }\log T + \frac{\beta_2 \rho^{2(k-1)N}  d(\hat \Phi_0, \Phi_\star)^2}{\sigma_{k-1}^2}  + \beta_2 \frac{C_{\mathsf{contract}}^2 \pzeromax N \log^2 (HT)}{(1-\sqrt{2}\rho^N)^2 H \tau_{k-1} \sigma_{k-1}^2}} \\&  
        + \tau_{k-1} \du \norm{P_{\star}^{(h)}} \sigma_k^2 \Psi_{B_\star^{(h)}}^2 + x_b \log T \norm{P_\star^{(h)}}. 
    \end{align*}
    Substituting our choice of $\sigma_k^2$ from \eqref{eq: exploration easy to learn}, we have
    \begin{align*}
        R_1^{(h)} - T \calJ(K_\star^{(h)}) 
        &\lesssim
        \sum_{k=2}^{k_{\fin}}  \frac{ \norm{P_\star^{(h)}}^8 \sigma^2  \dtheta   }{  \alpha^2 }\log T + \paren{\beta_2 \norm{P_\star^{(h)}}^8 + \du \norm{P_{\star}^{(h)}} \Psi_{B_\star^{(h)}}^2} \tau_{k-1} \rho^{(k-1)N}  d(\hat \Phi_0, \Phi_\star) \\
        & + \paren{\beta_2 \norm{P_\star^{(h)}}^8 \frac{C_{\mathsf{contract}}^2 \pzeromax N \log^2 (HT)}{(1-\sqrt{2}\rho^N)^2 \sqrt{H} } + \frac{\du \norm{P_{\star}^{(h)}} \Psi_{B_\star^{(h)}}^2}{\sqrt{H}} }\sqrt{\tau_{k-1} } \\&  + x_b \log T \norm{P_\star^{(h)}}\\
         &\lesssim \frac{ \norm{P_\star^{(h)}}^8 \sigma^2  \dtheta   }{  \alpha^2 }\log^2 T + \frac{\tau_1 \paren{\beta_2 \norm{P_\star^{(h)}}^8 + \du \norm{P_{\star}^{(h)}} \Psi_{B_\star^{(h)}}^2} d(\hat \Phi_0, \Phi_\star)}{1-\sqrt{2}\rho^N} +  x_b \log^2 T \norm{P_\star^{(h)}} \\
         & + \paren{\beta_2 \norm{P_\star^{(h)}}^8 \frac{C_{\mathsf{contract}}^2 \pzeromax N \log^2 (HT)}{(1-\sqrt{2}\rho^N)^2 }+ \du \norm{P_{\star}^{(h)}} \Psi_{B_\star^{(h)}}^2 }\frac{\sqrt{T}}{\sqrt{H}}.
    \end{align*}
    We conclude by substituting $\beta_2$ from Assumption~\ref{asmp: upper bound on representation error no exp}, $C_{\mathsf{contract}}$ from \Cref{thm: dfw bound formal}, and $\tau_1$ from \Cref{thm: regret bound no exploration}. 
\end{proof}

With these lemmas in hand, we are now ready to prove the main results. 

\subsubsection{Proof of \Cref{thm: regret bound naive exploration}}

\begin{proof}
It follows from Lemma 19 of \cite{lee2023nonasymptotic} that 
    \begin{align}
        \label{eq: R3 bound}
        R_3^{(h)} \leq 3 \tau_1 \max\curly{\dx, \du} \norm{P_{K_0^{(h)}}} \Psi_{B_\star^{(h)}}^2.
    \end{align}
The second term, $R_2^{(h)}$ may be bounded by using the fact that the state is bounded up until a failure situation is reached, and after that failure situation, the initial stabilizing controller is played. In the  probability $1-T^{-2}$, we have from Lemma 20 of \cite{lee2023nonasymptotic} that
\begin{align}
    \label{eq: R2 bound}
        R_2^{(h)} \leq T^{-1} \log(T) \texttt{poly}\paren{\sigma, \dx, \du, \dtheta, x_b, K_b, \norm{Q}, \thetamax, \pzeromax, \psibmax} + \sum_{k=1}^{k_{\mathsf{fin}}} 2(\tau_k - \tau_{k-1}) \du \sigma_k^2.
\end{align}
By substituting the choice of $\sigma_k^2$ from \Cref{thm: regret bound naive exploration} into the above inequality, and invoking \Cref{lem: success event regret bound 1}, we find that 
\begin{align*}
    \calR_T^{(h)} &= R_1^{(h)} - T\calJ^{(h)}(K_\star^{(h)}) + R_2^{(h)} + R_3^{(h)} \\
    & \leq \texttt{poly}\paren{\sigma, \dx, \du, \dtheta, x_b, K_b, \norm{Q}, \thetamax, \pzeromax, \psibmax, \tau_{\mathsf{warm\,up}}, x_b, \frac{1}{1-\sqrt{2}\rho^N}, d(\hat \Phi_0,\Phi_\star), \log H} \log^9 T \\
        &+ \texttt{poly}\paren{\pzeromax, \psibmax, \sigma}\sqrt{\dtheta \du} \sqrt{T} \log^2 T \\
        &+ \texttt{poly}\paren{\dx, \du, \dtheta, \pzeromax, \psibmax, \thetamax, \sigma, \frac{1}{1-\sqrt{2}\rho^N}, N}\frac{T^{3/4}}{H^{1/5}} \log^2(TH).
\end{align*}
\end{proof}

\subsubsection{Proof of \Cref{thm: regret bound no exploration}}

\begin{proof}

We may again invoke Lemma 19 and 20 of \cite{lee2023nonasymptotic} to show that \eqref{eq: R2 bound} and \eqref{eq: R3 bound} hold. Substituting the choice of $\sigma_k^2$ from \Cref{thm: regret bound no exploration} into \eqref{eq: R2 bound}, and invoking  \Cref{lem: success event regret bound 2}, we find
\begin{align*}
    \calR_T^{(h)} &= R_1^{(h)} - T\calJ^{(h)}(K_\star^{(h)}) + R_2^{(h)} + R_3^{(h)} \\
    & \leq \texttt{poly}\paren{\sigma, \dtheta, \du, \dx, \frac{1}{\alpha}, \frac{1}{1-\sqrt{2}\rho^N}, \psibmax, \pzeromax, x_b, K_b, \thetamax, \norm{Q}, \tau_{\mathsf{warm\,up}}, d(\hat \Phi_0, \Phi_\star)} \log^4 T \\
         & + \texttt{poly}\paren{\sigma, \dtheta, \du, \dx, \frac{1}{\alpha}, \frac{1}{1-\sqrt{2}\rho^N}, \psibmax, \pzeromax, x_b,  N} \frac{\sqrt{T}}{\sqrt{H}} \log^2(TH).
\end{align*}
\end{proof}

\end{document}